%% file: main.tex
\documentclass[final,12pt]{colt2024} 

\input{header}
\input{def}

\title[High-probability bounds for TD learning]{Improved High-Probability Bounds for the Temporal Difference Learning Algorithm via Exponential Stability}
\usepackage{times}

\coltauthor{%
 \Name{Sergey Samsonov} \Email{svsamsonov@hse.ru}\\
 \addr HSE University, Moscow, Russia %
 \AND
 \Name{Daniil Tiapkin} \Email{daniil.tiapkin@polytechnique.edu}\\
 \addr Centre de Math{\'e}matiques Appliqu{\'e}es -- CNRS -- {\'E}cole polytechnique -- Institut Polytechnique de Paris,
 France \\
Universit{\'e} Paris-Saclay, CNRS, Laboratoire de math{\'e}matiques d'Orsay, France 
 \AND
 \Name{Alexey Naumov} \Email{anaumov@hse.ru}\\
 \addr HSE University, Moscow, Russia\\
 Steklov Mathematical Institute of Russian Academy of Sciences, Moscow, Russia
  \AND
 \Name{Eric Moulines} \Email{eric.moulines@polytechnique.edu}\\
 \addr Centre de Math{\'e}matiques Appliqu{\'e}es -- CNRS -- {\'E}cole polytechnique -- Institut Polytechnique de Paris,
 France \\
    Mohamed bin Zayed University of Artificial Intelligence, UAE
}

\begin{document}

\maketitle

\begin{abstract}%
In this paper we consider the problem of obtaining sharp bounds for the performance of temporal difference (TD) methods with linear function approximation for policy evaluation in discounted Markov decision processes. We show that a simple algorithm with a universal and instance-independent step size together with Polyak-Ruppert tail averaging is sufficient to obtain near-optimal variance and bias terms. We also provide the respective sample complexity bounds. Our proof technique is based on refined error bounds for linear stochastic approximation together with the novel stability result for the product of random matrices that arise from the TD-type recurrence.
\end{abstract}


\section{Introduction}
\input{intro.tex}

\section{General LSA results}
\label{sec:lsa}
\input{results_lsa.tex}

\section{TD learning under i.i.d. noise}
\label{sec:td_learning}
\input{td_learning.tex}

\section{On optimality of TD(0) for i.i.d. sampling scheme}
\label{sec:lower_bounds}
\input{td_lower.tex}

\section{TD learning under Markov noise}
\label{sec:markov_td}
\input{markov_main.tex}

\section{Conclusion}
\label{sec:conclusion}
\input{conclusion.tex}

\acks{The work of Sergey Samsonov and Alexey Naumov was prepared within the framework of the HSE University Basic Research Program. The work of D. Tiapkin has been supported by the Paris Île-de-France Région in the framework of DIM AI4IDF. The work by Eric Moulines is partially funded by the European Union (ERC-2022-SYG-OCEAN-101071601). Views and opinions expressed are however those of the author(s) only and do not necessarily reflect those of the European Union or the European Research Council Executive Agency. Neither the European Union nor the granting authority can be held responsible for them. This work was partially conducted under the auspices of the Lagrange Mathematics and Computing Research Center.}

\newpage 
\bibliography{references}

\newpage 
\appendix
\input{supplementary-material}

\crefalias{section}{appendix} 

\end{document}

%% file: header.tex
\usepackage{amsfonts,amsmath,amssymb}
\usepackage{xcolor}
\usepackage{hyperref}
\usepackage{srcltx}
\usepackage{etoolbox}
\usepackage{nameref}
\usepackage{comment}
\usepackage{mathrsfs}
\usepackage{enumerate}
\usepackage[shortlabels]{enumitem}
\usepackage{amsfonts} 
\usepackage{nicefrac} 
\usepackage{mathtools}
\usepackage{dsfont,mathrsfs}
\usepackage{cleveref}

\usepackage{multirow}
\usepackage{colortbl}
\definecolor{bgcolor}{rgb}{0.8,1,1}
\definecolor{bgcolor2}{rgb}{0.8,1,0.8}
\definecolor{niceblue}{rgb}{0.0,0.19,0.56}
\usepackage{threeparttable}

\usepackage{tcolorbox}
\usepackage{pifont}
\definecolor{mydarkgreen}{RGB}{39,130,67}
\definecolor{mydarkred}{RGB}{192,47,25}
\newcommand{\green}{\color{mydarkgreen}}
\newcommand{\red}{\color{mydarkred}}
\newcommand{\cmark}{{\green\ding{51}}}
\newcommand{\xmark}{{\red\ding{55}}}

\usepackage{xcolor}
\definecolor{dmorange500}{HTML}{FF5F19}
\definecolor{dmblue300}{HTML}{2267EB}
\definecolor{dmred300}{HTML}{FF617B}
\hypersetup{
	colorlinks,
	citecolor=dmblue300
}

\usepackage{xargs}
\usepackage{multirow}

\usepackage{aliascnt}
\usepackage{autonum}
\usepackage{upgreek}

\newtheorem{assum}{A\hspace{-2pt}}

\crefname{theorem}{theorem}{Theorems}
\Crefname{theorem}{Theorem}{Theorems}

\newcounter{lemma}
\renewcommand{\thelemma}{\arabic{lemma}}

\renewenvironment{lemma}[1][]{
    \refstepcounter{lemma}
    \noindent\textbf{Lemma~\thelemma. #1}\itshape
}{\par}

\crefname{lemma}{lemma}{lemmas}
\Crefname{lemma}{Lemma}{Lemmas}

\crefname{remark}{remark}{remarks}
\Crefname{remark}{Remark}{Remarks}

\newcounter{corollary}
\renewcommand{\thecorollary}{\arabic{corollary}}

\renewenvironment{corollary}[1][]{
    \refstepcounter{corollary}
    \noindent\textbf{Corollary~\thecorollary. #1}\itshape
}{\par}

\crefname{corollary}{corollary}{corollaries}
\Crefname{corollary}{Corollary}{Corollaries}

\crefname{proposition}{proposition}{propositions}
\Crefname{proposition}{Proposition}{Propositions}

\crefname{definition}{definition}{definitions}
\Crefname{Definition}{Definition}{Definitions}

\crefname{example}{example}{examples}
\Crefname{Example}{Example}{Examples}

\crefname{figure}{figure}{figures}
\Crefname{Figure}{Figure}{Figures}

\crefname{table}{table}{tables}
\Crefname{Table}{Table}{Tables}

\crefname{algorithm}{algorithm}{algorithms}
\Crefname{Algorithm}{Algorithm}{Algorithms}

\crefname{assum}{A\hspace{-3pt}}{A\hspace{-3pt}}
\crefname{assumb}{B\hspace{-2pt}}{B\hspace{-2pt}}
\crefname{assumUGE}{UGE\hspace{-1pt}}{UGE\hspace{-1pt}}
\crefname{assumID}{IND\hspace{-1pt}}{IND\hspace{-1pt}}
\crefname{assumUE}{UE\hspace{-1pt}}{UE\hspace{-1pt}}
\crefname{assumM}{M\hspace{-1pt}}{M\hspace{-1pt}}

\newlist{renumerate}{enumerate}{3}
\setlist[renumerate]{wide, labelwidth=!, labelindent=0pt,label=(\roman*)}

\newlist{aenumerate}{enumerate}{3}
\setlist[aenumerate]{wide, labelwidth=!, labelindent=0pt,label=(\arabic*)}

\newlist{aaenumerate}{enumerate}{3}
\setlist[aaenumerate]{wide, labelwidth=!, labelindent=0pt,label=(\alph*)}

\newlist{aenumerateSpace}{enumerate}{3}
\setlist[aenumerateSpace]{wide, labelwidth=!,label=(\arabic*)}

\newlist{benumerate}{enumerate}{3}
\setlist[benumerate]{wide, labelwidth=!, labelindent=0pt,label=$\bullet$}

%% file: def.tex
\def\supconsteps{\supnorm{\funnoisew}}

\newcommandx\conststab[1][1=p]{\varkappa_{#1}}

\newcommand{\PE}{\mathbb{E}}
\newcommand{\PP}{\mathbb{P}}
\newcommandx{\genericb}[1][1=]{b_{#1}}
\newcommandx{\Constros}[1][1=]{\operatorname{C}_{\operatorname{Ros},#1}}
\newcommandx{\Constburk}[1][1=]{\operatorname{C}_{\operatorname{Burk}}}
\newcommandx{\driftW}[1][1=]{W_{#1}}

\def\metricz{\mathsf{d}_{\Zset}}
\def\metrics{\mathsf{d}_{\S}}
\newcommandx{\metricd}[1][1=]{\mathsf{d}_{#1}}

\newcommandx\invmeasure[1][1=]{\Pi_{#1}}
\newcommandx{\PPjoint}[1][1=]{\PP^{\MKjoint[#1]}}
\newcommandx{\PEjoint}[1][1=]{\PE^{\MKjoint[#1]}}
\newcommandx{\PEMID}[1][1=\alpha]{\PE^{\MK[#1]}}
\newcommandx{\PPMID}[1][1=\alpha]{\PP^{\MK[#1]}}

\newcommand{\supnorm}[1]{\norm{ #1 }[\infty]}
\newcommandx{\MKjoint}[1][1=]{\bar{\operatorname{P}}_{#1}}
\newcommandx\costw[1][1=]{\mathsf{c}_{#1}}

\newcommandx\Intergrdist[1][1=]{\mathbb{M}_{1}(#1)}

\newcommandx{\mmarkov}[1][1=0]{m^{(\Markov)}_{#1}}

\def\Q{\operatorname{Q}}

\renewcommand{\S}{\mathcal{S}}
\newcommand{\A}{\mathcal{A}}
\def\diag{\operatorname{diag}}
\def\PMDP{P}

\def\Zset{\mathsf{Z}}
\def\Zsigma{\mathcal{Z}}

\def\rset{\mathbb{R}}

\def\nset{\ensuremath{\mathbb{N}}}
\def\nsets{\ensuremath{\mathbb{N}}}

\newcommand{\msi}{\mathsf{I}}
\newcommand{\msj}{\mathsf{J}}

\def\MatB{B}

\newcommand{\bConst}[1]{\operatorname{C}_{{\mathbf{#1}}}}

\def\thetainit{\theta_0}

\newcommandx\sequence[4][2=,3=,4=]
{\ifthenelse{\equal{#3}{}}{\ensuremath{\{ #1_{#2 #4}\}}}{\ensuremath{\{ #1_{#2 #4} \}_{#2 \in #3}}}}

\newcommandx\sequenceD[2][2=]
{\ifthenelse{\equal{#2}{}}{\ensuremath{\{ #1\}}}{\ensuremath{\{ #1\!~:\!~#2  \}}}}
\newcommandx\sequenceDouble[4][3=,4=]
{\ifthenelse{\equal{#3}{}}{\ensuremath{\{ (#1_{#3},#2_{#3})\}}}{\ensuremath{\{ (#1_{#3},#2_{#3})\}_{#3 \in #4}}}}

\newcommandx{\sequencen}[2][2=n\in\nset]{\ensuremath{\{ #1, \eqsp #2 \}}}
\newcommandx\sequencens[2][2=n]
{\ensuremath{\{ #1_{#2} \!~:\!~#2\in\nsets\}}}
\newcommandx\sequencet[4]
{\ensuremath{\{ #1{#2_{#3}} \, : \, \eqsp #3 \in #4 \}}}
\def\PE{\mathbb{E}}
\def\ProdB{\Gamma}

\def\ProdBa{\ProdB^{(\alpha)}}

\newcommandx{\PVar}[1][1=]{\ensuremath{\operatorname{Var}_{#1}}}

\def\noisecov{\Sigma_\varepsilon}
\def\noisecovnew{\Sigma^{(tr)}_\varepsilon}
\def\noisecovopt{\Sigma^{(opt)}_\varepsilon}
\def\tracebound{\varrho}
\def\noisecovtd{\Sigma^{(TD)}_\varepsilon}
\def\tracebound{\varrho}

\def\covfeat{\Sigma_\varphi}

\newcommandx{\MK}[1][1=\alpha]{\mathrm{P}_{#1}}
\newcommandx\MKK[1][1=\alpha]{\mathrm{K}_{#1}}
\def\MKQ{\mathrm{Q}}

\newcommandx{\PEtilde}[1][1=]{\PE^{\mathrm{K}_{#1}}}
\newcommandx{\PPtilde}[1][1=]{\PP^{\mathrm{K}_{#1}}}
\newcommand{\PEext}{\tilde{\PE}}
\newcommand{\PPext}{\tilde{\PP}}

\def\tildeo{\tilde{\mathcal{O}}}

\newcommandx{\norm}[2][2=]{\Vert#1 \Vert_{{#2}}}
\newcommandx{\normLigne}[2][2=]{\Vert#1 \Vert_{{#2}}}
\newcommandx{\normLine}[2][2=]{\Vert#1 \Vert_{{#2}}}
\newcommandx{\normop}[2][2=]{\Vert{#1}\Vert_{{#2}}}
\newcommandx{\normopLigne}[2][2=]{\Vert{#1}\Vert_{{#2}}}
\newcommandx{\normopLine}[2][2=]{\Vert{#1}\Vert_{{#2}}}
\newcommandx{\osc}[2][1=]{\mathrm{osc}_{#1}(#2)}

\newcommandx{\normlip}[2][2=\operatorname{Lip}]{\Vert#1 \Vert_{{#2}}}
\newcommand{\lip}{\operatorname{L}}
\newcommandx{\lipspace}[1]{\lip_{#1}}

\newcommandx{\CPP}[3][1=]
{\ifthenelse{\equal{#1}{}}{{\mathbb P}\left(\left. #2 \, \right| #3 \right)}{{\mathbb P}_{#1}\left(\left. #2 \, \right | #3 \right)}}
\newcommandx{\CPPtilde}[3][1=]
{\ifthenelse{\equal{#1}{}}{{\tilde{\mathbb P}}\left(\left. #2 \, \right| #3 \right)}{{\tilde{\mathbb P}}_{#1}\left(\left. #2 \, \right | #3 \right)}}

\def\iid{i.i.d.}

\newcommandx{\as}[1][1=\PP]{\ensuremath{#1\, -\mathrm{a.s.}}}

\newcommand{\eqsp}{\;}

\newcommand{\Id}{\mathrm{I}}
\def\prtheta{\bar{\theta}}

\def\utheta{\tilde{\theta}^{\mathsf{(tr)}}}
\def\vtheta{\tilde{\theta}^{\mathsf{(fl)}}}

\newcommandx{\boundmetric}[1][1=]{\kappa_{\MKK[#1]}}

\newcommand{\coint}[1]{\left[#1\right)}

\newcommand{\ocint}[1]{\left(#1\right]}

\newcommand{\ccint}[1]{\left[#1\right]}

\newcommandx{\Nnorm}[2][1=V]{[ #2]_{#1}}
\newcommandx{\lipnorm}[2][1=g]{[ #1]_{#2}}

\newcommandx{\CPE}[3][1=]{{\mathbb E}^{#3}_{#1}\left[#2\right]}
\newcommandx{\CPEext}[3][1=]{\tilde{\mathbb E}^{#3}_{#1}\left[#2\right]}
\newcommandx{\CPEtilde}[3][1=]{{\tilde{\mathbb E}}^{#3}_{#1}\left[#2\right]}
\newcommandx{\CPEs}[3][1=]{{\mathbb E}^{#3}_{#1}[#2]}

\def\thetalim{\theta_\star}

\newcommand{\trace}[1]{\operatorname{Tr}(#1)}
\newcommand{\term}[2]{\mathsf{R}_{#2}(#1)}

\newcommand{\tvnorm}[1]{\left\Vert #1 \right\Vert_{\mathrm{TV}}}

\newcommand{\rme}{\mathrm{e}}
\newcommand{\rmd}{\mathrm{d}}
\def\funcAw{\mathbf{A}}
\def\funcBw{\mathbf{B}}
\newcommand{\funcA}[1]{\funcAw(#1)}

\def\funcbw{\mathbf{b}}
\newcommand{\funcb}[1]{\funcbw(#1)}
\newcommandx{\zmfuncA}[2][1=]{\tilde{\funcAw}^{#1}(#2)}
\newcommandx{\zmfuncAw}[1][1=]{\tilde{\funcAw}_{#1}}
\newcommandx{\zmfuncb}[2][1=]{\tilde{\funcbw}^{#1}(#2)}
\def\funnoisew{\varepsilon}
\newcommand{\funcnoise}[1]{\funnoisew(#1)}

\newcommandx{\funcct}[2][1=]{\funcctilde^{#1}(#2)}

\def\qcond{\kappa_{\Q}}

\def\State{Z}

\def\taumix{t_{\operatorname{mix}}}

\newcommandx{\CovC}[1][1=u]{\operatorname{C}_{#1}}

\def\red{\color{red}}

\def\msa{\mathsf{A}}
\def\msb{\mathsf{B}}
\def\msz{\mathsf{Z}}
\def\mcz{\mathcal{Z}}
\newcommand\borel[1]{\mathcal{B}(#1)}

\DeclareMathAlphabet{\mathpzc}{OT1}{pzc}{m}{it}

\def\lyapW{\mathpzc{W}}

\newcommandx{\bias}[1][1=\alpha]{\operatorname{B}_{#1}}

\newcommandx\probaMarkovTilde[2][2=]
{\ifthenelse{\equal{#2}{}}{{\widetilde{\mathbb{P}}_{#1}}}{\widetilde{\mathbb{P}}_{#1}\left[ #2\right]}}

\def\mcf{\mathcal{F}}

\newcommand{\parentheseDeuxLigne}[1]{[ #1 ]}

\def\bA{\bar{\mathbf{A}}}

\def\X{{\mathbf{X}}}
\def\Y{{\mathbf{Y}}}

\def\thetas{\thetalim}

\def\Am{{\mathbf{A}}}
\def\bm{{\mathbf{b}}}
\def\funcctilde{\tilde{c}_u}

\def\barb{\bar{\mathbf{b}}}
\newcommandx{\driftb}[1][1=p]{\bar{b}_{#1}}

\def\Zbf{\mathbf{Z}}

\def\red{\color{red}}
\newcommandx{\boldb}[1][1={q}]{\mathsf{b}_{#1}}

\newcommandx{\ConstGW}[1][1={n,\lyapW}]{\operatorname{G}_{#1}}

\newcommandx{\ConstMW}[1][1={n,\lyapW}]{\operatorname{M}_{#1}}

\newtheorem{assumptionC}{\textbf{C}\hspace{-1pt}}
\Crefname{assumptionC}{\textbf{C}\hspace{-1pt}}{\textbf{C}\hspace{-1pt}}
\crefname{assumptionC}{\textbf{C}}{\textbf{C}}

\newtheorem{assumTD}{\textbf{TD}\hspace{-1pt}}
\Crefname{assumTD}{\textbf{TD}\hspace{-1pt}}{\textbf{TD}\hspace{-1pt}}
\crefname{assumTD}{\textbf{TD}}{\textbf{TD}}

\newtheorem{assumptionM}{\textbf{UGE}\hspace{-1pt}}
\Crefname{assumptionM}{\textbf{UGE}\hspace{-1pt}}{\textbf{UGE}\hspace{-1pt}}
\crefname{assumptionM}{\textbf{UGE}}{\textbf{UGE}}

\def\distance{\mathsf{d}}

\newcommandx{\vartconstwas}[1][1=V]{c_{#1}}

\newcommandx{\deltawas}[1][1=*]{\delta_{#1}}

\newcommandx{\wasser}[4][1=\distance,4=]{\mathbf{W}_{#1}^{#4}\left(#2,#3\right)}
\newcommandx{\covcoeff}[2]{\rho_{#1}^{(#2)}}

\newcommand{\dobrush}{\mathsf{\Delta}}
\newcommandx{\dobru}[3][1=,3=]{\dobrush_{#1}^{#3}( #2)}  

\def\invariantQ{\mu}

\def\tZs{\tilde{Z}^{\star}}
\def\tZ{\tilde{Z}}
\def\tmszn{\tilde{\msz}_{\nset}}
\def\tmczn{\tilde{\mcz}_{\nset}}
\def\qexponent{q}
\def\ppexponent{p}
\def\Markov{\mathrm{M}}

\newcommand{\PEcoupling}[2]{\tilde{\PE}_{#1,#2}}
\newcommand{\PPcoupling}[2]{\tilde{\PP}_{#1,#2}}

\newcommandx{\dlim}[1]{\ensuremath{\stackrel{#1}{\Longrightarrow}}}
\def\G{\mathbf{G}}
\def\matrbound{\mathsf{g}}
\def\randmbound{\mathsf{\omega}} 

%% file: intro.tex
This paper aims to provide sharp statistical guarantees for the temporal difference (TD) learning algorithms that use a linear function approximation in the on-policy setting. The TD algorithm \citep{sutton1988learning,sutton:book:2018} is one of the most fundamental methods for policy evaluation in reinforcement learning (RL), acknowledged for its simplicity and ease of implementation. Theoretical analysis of TD learning in general state space is usually performed in the setting of \emph{linear function approximation} \citep{bertsekas1996neuro}. The asymptotic convergence of TD in such a setting was shown in \citep{tsitsiklis:td:1997}. At the same time, the current trend in the field of stochastic approximation is to study non-asymptotic properties of the error, and provide high probability error bounds \citep{mou2020linear, durmus2022finite, huo2023bias}. However, many of the existing works \citep{bhandari2018finite, dalal:td0:2018,lakshminarayanan:2018a} characterize the convergence guarantees and sample complexity only in terms of the mean-squared error (MSE). Other works \citep{korda2015td, patil2023finite} study versions of the TD learning algorithm with projections in order to overcome the crucial problem in the analysis related to the stability of the random matrix products \citep{guo1994stability,guo1995exponential}. The latter projections onto the feasible set are usually impractical. Other works provide the high-probability bounds \citep{li2023sharp}, but with the choice of step sizes relying on the (unknown in practice) instance-dependent quantities, related to the problem design, see \Cref{sec:td_learning} for details.

\paragraph{Contributions. } We enhance the existing $p$-moment and high probability bounds for the iterates of the TD learning procedure. Towards this aim we follow the framework of the linear stochastic approximation (LSA). Our main contributions are as follows:
\begin{itemize}[noitemsep,topsep=0pt]
    \item We propose a refined high-probability error bound for TD learning with linear function approximation and Polyak-Ruppert averaging with a \emph{universal and instance-independent} step size. We consider both the generative model assumption and trajectory-wise evaluation based on a sequence of observations forming a Markov chain. However, we show that the variance term of the instance-independent TD learning might be suboptimal with respect to its dependence upon the properties of the feature map.
    \item In order to obtain our results for the particular setting of TD learning, we provide error bounds for the LSA algorithms by directly assuming the core \emph{exponential stability} of the random matrix product (see assumption \Cref{assum:exp_stability} and related discussion). We then present a novel proof of exponential stability specifically for the TD(0) algorithm, which quantifies the speed at which the algorithm forgets its initial error. Our bound is tighter than the previously known results in the literature and serves as a pivotal element in eliminating the need for an additional projection step when addressing the high-order moments of the error \citep{patil2023finite}. Conventional proofs for the exponential stability of matrix products often impose an unnecessary restriction on the choice of step size by adjusting it to the \emph{minimal} eigenvalue of the design matrix. This limitation explains the need for projections in \citep{patil2023finite} and the instance-dependent step size in \citep{li2023sharp}. Our approach allows us to mitigate this drawback.
\end{itemize}
\paragraph{Related works. }
The number of contributions to the analysis of TD learning is substantial, and we can not hope to comprehensively cover them all. Significant progress has been made in evaluating the effects of tolerance levels and various parameters on the sampling efficiency of TD learning with linear function approximation \citep{lakshminarayanan:2018a,dalal:td0:2018,bhandari2018finite,srikant2019finite}. However, the minimax-optimal dependence on the tolerance level has only been established in expectation, see \citep{li_accelerated_TD}. This paper considered the MSE guarantees. The authors in \citep{khamaru2021temporal} considered complexity bounds for TD learning for finite state space in terms of $\ell_{\infty}$-norm. A recent paper \citep{pmlr-v211-duan23a} focuses on multi-step ahead TD learning. Among the closest counterparts to our paper, we must mention the following:

\begin{itemize}[noitemsep,topsep=0pt]
\item \citep{li_accelerated_TD} establishes lower bounds on the mean squared error (MSE) for policy evaluation problems. They also present bounds on the MSE of the variance-reduced TD learning algorithm, which covers both generative model and Markov sampling methods.
\item \citep{li2023sharp} provides high-probability bounds and sample complexity for the TD(0) learning algorithm and extends these findings to its off-policy counterpart (TDC) under the i.i.d. sampling assumption. Despite not separating the respective error bounds into the deterministic and stochastic components, the authors in \citep{li2023sharp} consider the step size that scales with the minimal eigenvalue of the feature matrix (see \Cref{assum:feature_design} for details). Such scaling is not only a drawback for the practical implementation of the algorithm but also inevitably implies a suboptimal rate of forgetting the initial error.
\item \citep{patil2023finite} focuses on determining the bounds of the second moment for TD(0) and high-probability bounds for projected TD(0) iterates. However, the established high-probability bounds require a projection procedure that relies on prior knowledge of the true parameter norm $\norm{\thetas}$, which is impractical. Despite this limitation, the study shows that this problem can be resolved using the restart technique.
\end{itemize}

Non-asymptotic results for TD learning could be also derived from the analysis of the general LSA algorithms \citep{mou2020linear,durmus2022finite}. However, the respective error bounds are typically loose in terms of problem-dependent quantities, related to the feature mapping considered in TD with linear function approximation. Detailed discussion is provided after \Cref{assum:exp_stability}.

\paragraph{Notations. }
\label{par:notations}
For the sequences $(a_n)_{n \in \nset}$ and $(b_n)_{n \in \nset}$ we write $a_n \lesssim b_n$ if there exist an absolute constant $c > 0$, such that $a_n \leq c b_n$ for any $n \in \nset$. We also write that $a_n = \tildeo(b_n)$, if $a_n \leq c (\log n)^{\kappa} b_n$ for some $\kappa > 0$. For the matrix $A \in \rset^{d \times d}$, such that $A = A^{\top} \succeq 0$,  and vector $x \in \rset^{d}$ we define the corresponding $A$-norm of $x$ as $\norm{x}[A] = \sqrt{x^{\top} A x}$. In the present text, the following abbreviations are frequently used: "w.r.t." stands for "with respect to", "i.i.d." stands for "independent and identically distributed".

%% file: results_lsa.tex
We consider the LSA problem, that is, we aim to solve a linear system $\bA \theta = \barb$ with a unique solution $\thetas$. We do not have access to $\bA$ and $\barb$  but instead we have access to a sequence of observations  $\{( \funcA{Z_n}, \funcb{Z_n})\}_{n \in \nset}$, where $(Z_k)_{k \in \nset}$ are noise variables taking values in a measurable space $(\msz,\mcz)$ and $\Am \colon \msz \to \rset^{d \times d}$, $\bm\colon \msz \to \rset^d$ are measurable functions. We consider the setting where $(Z_k)_{k \in \nset}$ is a sequence of \iid\ random variables with common distribution $\invariantQ$ satisfying
\begin{equation}
\textstyle 
\PE_{\mu}[ \funcA{Z_1} ] = \bA\eqsp, \text{ and } \PE_{\mu} [ \funcb{Z_1} ] = \barb\eqsp.
\end{equation}
For a fixed step size $\alpha > 0$, burn-in period $n_0 \in \nset$, and initialization $\theta_0 \in \rset^{d}$, we consider the sequences of LSA iterates $\{\theta_n \}_{n \in \nset}$ and its tail-averaged counterpart $\{ \prtheta_{n_0, n} \}_{n \geq n_0+1}$ given by
\begin{equation}
\label{eq:lsa}
\begin{split}
\textstyle \theta_{k} &= \theta_{k-1} - \alpha \{ \funcA{Z_k} \theta_{k-1} - \funcb{Z_k} \} \eqsp,~~ k \geq 1, \\
\textstyle \prtheta_{n_0,n} &= \textstyle (n-n_0)^{-1} \sum_{k=n_0}^{n-1} \theta_k \eqsp, ~~n \geq n_0+1 \eqsp.
\end{split}
\end{equation}
Unless explicitly stated, we set $n_0 = n/2$ and write $\prtheta_{n}$ instead of $\prtheta_{n/2,n}$. The sequence $\{ \prtheta_{n} \}$ corresponds to the Polyak-Ruppert averaged iterates; see \citep{ruppert1988efficient,polyak1992acceleration}. Using the definition \eqref{eq:lsa} and elementary algebra, we obtain
\begin{equation}
\label{eq:error_expansion_1_step}
\textstyle
\theta_{n} - \thetas = (\Id - \alpha \funcA{Z_n})(\theta_{n-1} - \thetas) - \alpha \funcnoise{\State_{n}}\eqsp,
\end{equation}
where the noise variable $\funcnoise{\cdot}$ is defined as
\begin{equation}
\label{eq:noise_eps_definition}
\funcnoise{z} = \zmfuncA{z} \thetas - \zmfuncb{z}\eqsp, \, \zmfuncA{z}=\funcA{z} - \bA\eqsp, \, \zmfuncb{z}= \funcb{z} - \barb\eqsp.
\end{equation}
The quantity $\funcnoise{\cdot}$ is crucial for our analysis, since in controls the noise level measured \emph{at the solution $\thetas$}. Note also that $\funcnoise{Z_i}$ are centered and denote by $ \noisecov$ the covariance matrix of $\funcnoise{Z_i}$, that is,
\begin{equation}
\label{eq:def_noise_cov}
\textstyle \noisecov = \PE[ \funcnoise{Z_1} \funcnoise{Z_1}^\top] \eqsp.
\end{equation}
Running the recurrence \eqref{eq:error_expansion_1_step}, we obtain the error decomposition
\begin{equation}
\label{eq:key_decomspotion_lsa}
\textstyle 
\theta_{n} - \thetas = \underbrace{\ProdBa_{1:n} \{ \theta_0 - \thetas \}}_{\utheta_{n}} - \underbrace{\alpha \sum\nolimits_{j=1}^n \ProdBa_{j+1:n} \funcnoise{Z_j}}_{\vtheta_{n}}\eqsp.
\end{equation}
In the above formula we introduced the \emph{product of random matrices}
\begin{equation}
\label{eq:definition-Phi} \textstyle
\ProdBa_{m:n}  = \prod_{i=m}^n (\Id - \alpha \funcA{Z_i} ) \eqsp, \quad m,n \in\nset, \quad m \leq n \eqsp,
\end{equation}
with the convention $\ProdBa_{m:n}=\Id$ for $m > n$. The error decomposition \eqref{eq:key_decomspotion_lsa} is essential for the analysis of LSA algorithms since it allows to split the LSA error into two parts; see, among many others, \citep{aguech2000perturbation,durmus2022finite}. The first, $\utheta_{n}$, reflects the rate at which the initial error of the procedure is forgotten, and the second, $\vtheta_{n}$, is responsible for the fluctuations of the LSA iterates around the solution $\thetas$. The analysis of both $\vtheta_{n}$ and $\vtheta_{n}$ crucially relies on the properties of the matrix product $\ProdBa_{m:n}$. In what follows, we present a verifiable set of conditions for the general tail-averaged LSA procedure. In \Cref{sec:td_learning} we then give a recipe for checking these assumptions for the family of TD algorithms. Our first set of assumptions is classical for LSA:
\begin{assum}
\label{assum:noise-level}
Random variables $(Z_k)_{k \in \nset}$ are \iid\ taking values in $(\msz,\mcz)$ with a distribution $\invariantQ$ satisfying $\PE [ \funcA{Z_1} ] = \bA$ and $\PE[ \funcb{Z_1} ] = \barb$. Moreover,
{
\begin{equation}
\supconsteps = \textstyle{\sup_{z \in \msz}\normop{\funcnoise{z}} < \infty}\eqsp, \quad
\bConst{A} = \textstyle{\sup_{z \in \msz} \normop{\funcA{z}} \vee \sup_{z \in \msz} \normop{\zmfuncA{z}} < \infty}\eqsp.
\end{equation}
}
\end{assum}
Assumption \Cref{assum:noise-level} was considered in several  papers, e.g. \citep{srikant2019finite, chen2020explicit}. Almost sure bounds for $\normop{\funcA{\cdot}}$ can be replaced by weaker moment-type bounds following the methods described in \citep{mou2020linear,durmus2021stability}. However, the applications of results with unconstrained noise, especially in the Markov noise setting of \Cref{sec:markov_td}, involves additional technical difficulties. For this reason, we refrain from relaxing the boundedness \Cref{assum:noise-level}. Now we come to the crucial assumption about the matrix product $\ProdBa_{1:n}$. Namely, we define the following family of \emph{exponential stability} assumptions for some $p \in [2, \infty)$:

\begin{assum}
\label{assum:exp_stability}($p$)
There exist $a > 0$, $\conststab > 0$,  $\alpha_{p,\infty} > 0$ (depending on $p$), such that $\alpha_{p,\infty} p \leq 1/2$, and for any $\alpha \in (0;\alpha_{p,\infty})$, $u \in \rset^{d}$, $n \in \nset$,
\begin{equation}
\label{eq:concentration_iid_matrix}
\PE^{1/p}\bigl[ \normop{\ProdBa_{1:n} u}^{p}\bigr]
\leq \conststab (1 - \alpha a)^{n} \norm{u}\eqsp.
\end{equation}
\end{assum}
Verifying the exponential stability assumption \Cref{assum:exp_stability} is crucial for studying properties of both $\vtheta_{n}$ and $\utheta_{n}$, see e.g. \citep{guo1995exponential, priouret1998remark}. For TD algorithms, exponential stability has been verified in \cite{patil2023finite} (for $p = 2$) and \citep{li2023sharp} (for arbitrary $p > 2$, but with suboptimal $\alpha_{p,\infty}$, see discussion in \Cref{sec:td_learning}). Note that \Cref{assum:exp_stability} can be verified under the classical stability conditions for linear systems. In particular, it is enough to assume \Cref{assum:noise-level} and additionally assume that the system matrix $-\bA$ is Hurwitz. The Hurwitzness of $-\bA$ is a necessary and sufficient condition for the exponential stability of the continuous-time ODE system $\dot{\theta_t} = \bA \theta_t$, see e.g. \citep{jacob2012linear}. Therefore, it is a standard condition for the exponential speed of forgetting the initial error $\utheta_{n}$, see \citep{mou2020linear,durmus2021tight}. However, such an assumption allows to prove a contraction
\begin{equation}
\label{eq:conraction_rate_q_norm}
\textstyle
\norm{\Id - \alpha \bA}[\Q]^{2} \leq 1 - \alpha \tilde{a}
\end{equation}
only in specific matrix $\Q$-norm, associated with the solution of the Lyapunov equation 
\begin{equation}
\label{eq:lyapunov}
\textstyle 
\bA^{\top} \Q + \Q \bA = -P\eqsp.
\end{equation}
Here the choice of the matrix $P = P^{\top} \succ 0 $ is an additional degree of freedom, with the default choice being $P = \Id$. In this case the factor $\tilde{a}$ in \eqref{eq:conraction_rate_q_norm} might be overly small, yielding suboptimal bias forgetting rate. Detailed discussion is provided in \Cref{sec:old_stability_bound}. In this paper we suggest a different view on the problem: we assume \Cref{assum:exp_stability} directly and then verify it for the particular example of TD learning with linear function approximation.
\subsection{Refined LSA results with \iid\, noise}
\label{subsec:LSA_refined}
In this section we provide general results for the tail-averaged LSA iterates, which can be viewed as simplified versions of \citep[Theorem~2]{durmus2022finite} and \citep[Theorem~3]{mou2020linear} with explicit dependence on the instance-dependent quantities, such as the contraction rate $a$. First, we give an elementary statement for the mean square error.
\begin{theorem}
\label{th:lsa_pr_2nd_moment_main}
Assume \Cref{assum:noise-level} and \Cref{assum:exp_stability}($2$). Then for any $n \geq 2$, $\alpha \in (0;\alpha_{2,\infty}]$, and $\theta_0 \in \rset^d$, it holds that
\begin{multline}
\label{eq:pr_theta_2nd_moment_bound}
\PE^{1/2}[\norm{\bA(\prtheta_{n} - \thetas)}^2] \lesssim  \frac{\sqrt{\trace{\noisecov}}}{n^{1/2}}\left(1 + \frac{\conststab[2] \bConst{A} \sqrt{\alpha}}{\sqrt{a}}\right) \\
+ \frac{\conststab[2] \sqrt{\trace{\noisecov}}}{\sqrt{\alpha a} n}
+  \conststab[2] (1 - \alpha a)^{n/2} \left(\frac{1}{\alpha n} + \frac{\bConst{A}}{\sqrt{\alpha a} n}\right)\norm{\theta_0 - \thetas}\eqsp.
\end{multline}
\end{theorem}
\textit{Proof sketch.} The proof relies on the summation by parts applied to the LSA error \eqref{eq:error_expansion_1_step}. This approach was previously applied in \citep{mou2020linear,durmus2022finite}, and yields
\begin{align}
\label{eq:pr_err_decompose_main}
\textstyle
\bA\left(\prtheta_{n} -\thetalim\right) =  2 (\alpha n)^{-1} (\theta_{n/2}-\theta_{n}) - 2n^{-1} \sum_{t=n/2}^{n-1} e\left(\theta_{t}, Z_{t+1} \right)\eqsp,
\end{align}
where we have defined $e(\theta,z) = \zmfuncA{z} \theta- \zmfuncb{z} = \funcnoise{z} + \zmfuncA{z} (\theta - \thetas)$. This transform justifies why it is convenient to state the bounds in terms of $\norm{\bA(\prtheta_{n} - \thetas)}$. The rest of the proof follows from the martingale structure of the term $\sum_{t=n/2}^{n-1} e\left(\theta_{t}, Z_{t+1} \right)$ w.r.t. filtration $\mathcal{F}_t = \sigma(Z_s, s \leq t)$. We also need to show the last iterate error bound $\PE^{1/2}[\norm{\theta_{n} - \thetas}^2] = \tildeo(\sqrt{\alpha})$, a standard result that was previously obtained in many papers on LSA, see e.g. \citep{dalal:td0:2018,bhandari2018finite}. This explains the factor $1 + \mathcal{O}(\sqrt{\alpha})$, which affects the leading term in \eqref{eq:pr_theta_2nd_moment_bound}. We provide the complete proof in \Cref{sec:general_lsa_proofs}, see \Cref{th:LSA_last_iterate_refined}. 
\unskip\nobreak\hfill $\square$
\par
Now we provide a $p$-moment bound. We assume that \Cref{assum:exp_stability}$(\ell)$ is satisfied for any $\ell \geq 2$, however, similar results could be obtained if \Cref{assum:exp_stability}$(p)$ holds only for a fixed parameter $2 \leq p < \infty$.
\begin{theorem}
\label{th:LSA_PR_error}
Suppose that assumptions \Cref{assum:noise-level} and \Cref{assum:exp_stability}($\ell$) hold for any $\ell \geq 2$. Then, for any $n \in \nset$, $p \geq 2$,  $\alpha \in \coint{0, \alpha_{p + \log n, \infty}}$, $\thetainit \in \rset^d$, we have
\begin{equation}
\label{eq:p_th_moment_LSA_bound}
\begin{split}
&\PE^{1/p}[\norm{\bA(\prtheta_{n} - \thetas)}^p] \lesssim \frac{p^{1/2}\sqrt{\trace{\noisecov}}}{n^{1/2}}\bigl(1 + \term{\alpha}{1} \bigr) + \frac{\conststab[p+\log{n}](1+\bConst{A}) p \supconsteps }{n} \\
&\qquad\qquad + \frac{p \conststab[p+\log{n}] \sqrt{ \trace{\noisecov}}}{n\sqrt{a}}\left(1 + \frac{1}{\sqrt{\alpha p}}\right) + \conststab[p+\log{n}] (1-\alpha a)^{n/2} \left(\frac{1}{\alpha n} + \frac{p \bConst{A}}{\sqrt{\alpha a} n}\right)\norm{\theta_0 - \thetas}\eqsp,
\end{split}
\end{equation}
where the term $\term{\alpha}{1}$ is given by
\[
\term{\alpha}{1} = \frac{ \conststab[p+\log{n}] \sqrt{\alpha p} \bConst{A}}{\sqrt{a}} + \frac{\conststab[p+\log{n}] \bConst{A} \alpha p \supconsteps}{\sqrt{\trace{\noisecov}}} \eqsp.
\]
\end{theorem}
\textit{Proof sketch.}
We use the same key decomposition \eqref{eq:pr_err_decompose_main} and utilize the martingale structure of $\sum_{t=n/2}^{n-1} e\left(\theta_{t}, Z_{t+1} \right)$ with Rosenthal's inequality \citep{pinelis_1994}. This technique requires to handle the (remainder w.r.t. $n$) term $\PE^{1/p}[\max_{t}\normop{\zmfuncAw[t+1](\theta_t - \thetas)}^{p}]$, which scales with $n$ as $n^{1/q}$ for any $q \geq p$ and $\alpha \in \coint{0, \alpha_{q, \infty}}$. This dependence is removed by setting $q= p + \log(n)$, and $\alpha \in \coint{0, \alpha_{p+\log n, \infty}}$. Complete proof is given in \Cref{sec:general_lsa_proofs}-\Cref{th:lsa_pr_error_p_moment}. 
\unskip\nobreak\hfill $\square$
\par 
The closest counterparts of \Cref{th:LSA_PR_error}, \citep[Theorem 2]{durmus2022finite}, and \citep[Theorem 3 and 4]{mou2020linear}, are less explicit in terms of dependence of the error terms upon the contraction rate $a$. Note that for a general SA problem the constant $\conststab[p+\log{n}]$ above might scale polynomially with $d$, see \citep{huang2020matrix}. Yet in particular applications $\conststab[p+\log{n}]$ might be \emph{dimension-free} and even independent from $p$, as we show in \Cref{sec:td_learning}. The bound given in \Cref{th:LSA_PR_error} highlights a remarkable property: the leading term of \eqref{eq:p_th_moment_LSA_bound} contains an additional multiplicative factor of
\[
\textstyle 
1 + \conststab[p+\log{n}] \sqrt{\alpha p} \bConst{A} /\sqrt{a} + \mathcal{O}(\alpha)\eqsp.
\]
If $\alpha$ is chosen so that the ratio $\alpha p / a= o(1)$, we achieve the 'optimal' sub-Gaussian leading term
\begin{equation}
\label{eq:opt_subgaus_leading}
\textstyle 
p^{1/2}\sqrt{\trace{\noisecov}} / n^{1/2}\eqsp.
\end{equation}
Optimality of the presented leading term is discussed in \citep{fort:clt:markov:2015,mou2020linear}. This is consistent with the findings from decreasing step size, which have ensured the attractiveness of Polyak-Ruppert algorithms, see e.g. \citep{bhandari2018finite}. On the other hand, in case of \emph{instance-independent} choice of step size $\alpha$ it is possible that the ratio $\alpha p / a$ is not small. In such a scenario the dominant term in \eqref{eq:p_th_moment_LSA_bound} could far exceed the optimum sub-Gaussian leading term \eqref{eq:opt_subgaus_leading}. Recent studies addressing constant step size SA schemes \citep{durmus2022finite, mou2020linear} circumvent this problem by adjusting the SA step, $\alpha$, relative to the time horizon $n$ as $\alpha = O(n^{- \kappa})$ for some $\kappa \in \ocint{0,1}$. However, this approach may result in a slower reduction of the initial error $\norm{\theta_0 - \thetas}$ and suboptimal instance-dependent second-order terms, a phenomenon observed in \cite{khamaru2021temporal} in case of TD learning.

%% file: td_learning.tex
\begin{table*}[t!]
\renewcommand{\arraystretch}{1.5}
    \centering
    \scriptsize
\caption{Summary of error bounds for TD(0) algorithm with linear functional approximation.}
    \label{tab:comparison0}   
    \scriptsize
\resizebox{\linewidth}{!}{
  \begin{threeparttable}
    \begin{tabular}{|c|c|c|c|c|c|c|}
    \cline{1-7}
    \textbf{Paper} & \begin{tabular}{@{}c@{}}
    \textbf{Algorithm}
    \\[-2mm]
    \textbf{type}
    \end{tabular}  & 
    \begin{tabular}{@{}c@{}}
    \textbf{step size}
    \\[-2mm]
    \textbf{schedule}
    \end{tabular} 
    & 
    \begin{tabular}{@{}c@{}}
    \textbf{Universal}
    \\[-2mm]
    \textbf{step size}
    \end{tabular} 
    & 
    \begin{tabular}{@{}c@{}}
    \textbf{Markovian}
    \\[-2mm]
    \textbf{data}
    \end{tabular} 
    &\begin{tabular}{@{}c@{}}
    \textbf{High-order}
    \\[-2mm]
    \textbf{bounds}
    \end{tabular} 
    &
    \begin{tabular}{@{}c@{}}
    \textbf{Not require}
    \\[-2mm]
    \textbf{projections}
    \end{tabular} 
    \\
    \hline
    \citep{bhandari2018finite}\tnote{{\color{blue}(1)}} & 
    Polyak-Ruppert & $1/\sqrt{n}$ & \cmark & \cmark & \xmark  & \xmark
    \\
    \hline
    \citep{dalal:td0:2018}\tnote{{\color{blue}(2)}} & 
     Last iterate & $1/k^{\varkappa}$ & \cmark & \xmark & \cmark  & \cmark
    \\
    \hline
    \citep{lakshminarayanan:2018a} & 
     Polyak-Ruppert & constant $\alpha$ & \cmark & \xmark & \xmark  & \cmark
     \\
    \hline
    \citep{patil2023finite}\tnote{{\color{blue}(3)}} & 
     Polyak-Ruppert & constant $\alpha$ & \cmark & \cmark & \cmark  & \xmark
     \\
     \hline
    \citep{li2023sharp} & 
     Polyak-Ruppert & constant $\alpha$ & \xmark & \xmark & \cmark  & \cmark
     \\
     \hline
    This paper & 
     Polyak-Ruppert & constant $\alpha$ & \cmark & \cmark & \cmark  & \cmark
     \\
    \hline
    \end{tabular}   
    \begin{tablenotes}  
    \item[] \tnote{{\color{blue}(1)}} \citep{bhandari2018finite} considers constant step size $\alpha = 1/\sqrt{n}$ with $n$ being total number of iterations and provide suboptimal MSE bound of order $\tildeo(1/\sqrt{n})$;
    \tnote{{\color{blue}(2)}} \citep{dalal:td0:2018} uses last iterate and decreasing step size schedule with $\alpha_{k} = 1/k^{\varkappa}$. Hence, the corresponding bias forgetting rate is sublinear, and the $n$-step MSE is of order $\tildeo(1/n^{\kappa})$;
    \tnote{{\color{blue}(3)}} \citep{patil2023finite} uses projections in order to prove the concentration bounds, moreover, the definition of the projection set involves unknown parameter $\thetas$. 
\end{tablenotes}    
\end{threeparttable}
}
\end{table*}

In this section we apply results of \Cref{sec:lsa} to the TD learning procedure. Namely, we consider a problem of estimating a value of the policy $\pi$ in a discounted MDP (Markov Decision Process) given by a tuple $(\S,\A,\PMDP,r,\gamma)$.
Here, $\S$ and $\A$ stand, respectively, for state and action spaces, and $\gamma \in (0,1)$ is a discount factor. We assume that $\S$ is a complete metric space equipped with a metric $\metrics$ and Borel $\sigma$-algebra $\borel{\S}$. $\PMDP$ stands for the transition kernel $\PMDP(B | s,a)$, which determines the probability of moving from state $s$ to a Borel set $B \in \borel{\S}$ when action $a$ is performed. For simplicity, the reward function $r\colon \S \times \A \to [0,1]$ is assumed to be deterministic. The policy $\pi(\cdot|s)$ is a distribution over the action space $\A$ corresponding to the agent's action preferences in state $s \in \S$. We aim to estimate the agent's \emph{value function}
\[
\textstyle{
V^{\pi}(s) = \PE[\sum_{k=0}^{\infty}\gamma^{k}r(s_k,a_k)|s_0 = s]}\eqsp,
\]
where $a_{k} \sim \pi(\cdot | s_k)$, and $s_{k+1} \sim \PMDP(\cdot | s_{k}, a_{k})$, for any $k \in \nset$. We define the transition kernel
\begin{equation}
\label{eq:transition_matrix_P_pi}
\textstyle \PMDP_{\pi}(B | s) = \int_{\A} \PMDP(B | s, a)\pi(\rmd a|s)\eqsp,
\end{equation}
which corresponds to the $1$-step transition probability from state $s$ to a set $B \in \borel{\S}$. The state space is arbitrary:  $\S$ may be finite, but with  $|\S| \gg 1$, or $\S \subset \rset^D$ may be uncountable. In this setting, it is a common option to consider the \emph{linear function approximation} of the value function $V^{\pi}(s)$, defined for $s \in \S$, $\theta \in \rset^{d}$, and feature mapping $\varphi\colon \S \to \rset^{d}$ as
\[
\textstyle
V_{\theta}^{\pi}(s) =  \varphi^\top(s) \theta\eqsp.
\]
Here $d$ is the dimension of the feature space. We consider $V_{\theta}^{\pi}(s)$ as an approximation to the true value function $V^{\pi}(s)$, and our goal is to find a parameter $\thetas$ that defines the best linear approximation of $V^{\pi}$ \citep{tsitsiklis:td:1997}. To properly define what it means, we introduce some notations, following  \citep{li2023sharp}. We denote by $\mu$ the invariant distribution over the state space $\S$ induced by the transition kernel $\PMDP_{\pi}(\cdot | s)$ in \eqref{eq:transition_matrix_P_pi}. Then we define $\thetas$  as a solution  
\begin{equation}
\label{eq:def:theta}
\textstyle 
\thetas = \arg\min_{\theta \in \rset^d}\PE_{\mu}\bigl[\left( V^{\pi}_{\theta}(s) - V^{\pi}(s) \right)^2\bigr]\eqsp.
\end{equation}
We define the \emph{design matrix} $\covfeat$ as
\begin{equation}
\textstyle
\covfeat = \PE_{\mu}[\varphi(s)\varphi(s)^{\top}] \in \rset^{d \times d}\eqsp.
\end{equation}
\!\!In the following, we are interested in minimizing the following distance between $\theta \in \rset^d$ and $\thetas$:
\[
\textstyle
\norm{\theta - \thetas}[\covfeat] = \PE^{1/2}_{\mu}\left[ (V^{\pi}_\theta(s) - V^{\pi}_{\thetas}(s))^2  \right]\eqsp.
\]
For the estimator $\hat{\theta}$ of $\thetas$, our primary concern is to control the error $\norm{\hat\theta - \thetas}[\covfeat]$ in two ways: firstly, by controlling its second moment $\PE[\norm{\hat\theta - \thetas}[\covfeat]^2]$, and secondly, by giving high-probability bound; available results are summarized in \Cref{tab:comparison0}. We consider the following assumptions on the generative mechanism and on the feature mapping $\varphi(\cdot)$:
\begin{assumTD}
\label{assum:generative_model}
Tuples $(s,a,s')$ are generated \iid with $s \sim \mu$, $a \sim \pi(\cdot|s)$, $s' \sim \PMDP(\cdot|s,a)$\eqsp.
\end{assumTD}

\begin{assumTD}
\label{assum:feature_design}
Matrix $\covfeat$ is non-degenerate with the minimal eigenvalue $\lambda_{\min} \equiv \lambda_{\min}(\covfeat)$. Moreover, the feature mapping $\varphi(\cdot)$ satisfies 
$\sup_{s \in \S} \|\varphi(s) \| \leq 1$.
\end{assumTD}
The generative model assumption \Cref{assum:generative_model} is used in many previous works; see, e.g.~\citep{dalal:td0:2018, li2023sharp, patil2023finite}. In  \Cref{sec:markov_td} we generalize this assumption to more realistic setting of on-policy evaluation over a single trajectory, where the induced LSA noise is Markovian.
\par 
In the setting of linear functional approximation the problem of estimating $V^{\pi}(s)$ reduces to the problem of estimating $\thetas \in \rset^{d}$, which can be done via the LSA procedure. It is known (see e.g. \citep{tsitsiklis:td:1997}), that optimal (in a sense of \eqref{eq:def:theta}) parameter $\thetas$ is a solution to the deterministic system$\bA \thetas = \barb$, where 
\begin{align}
\label{eq:system_matrix}
\bA &= \PE_{s \sim \mu, s' \sim \PMDP_{\pi}(\cdot|s)} [\phi(s)\{\phi(s)-\gamma \phi(s')\}^{\top}] \\
\barb &= \PE_{s \sim \mu, a \sim \pi(\cdot|s)}[\phi(s) r(s,a)]\eqsp.
\end{align}
In order to write the instance of the LSA algorithm for the system \eqref{eq:system_matrix}, we introduce the $k$-th step randomness $Z_k= (s_k, a_k, s'_{k})$. With slight abuse of notation, we write $\funcAw_{k}$ instead of $\funcA{Z_k}$, and $\funcbw_{k}$ instead of $\funcb{Z_k}$. Then the corresponding LSA update equation with step size $\alpha$ writes as
\begin{equation}
\label{eq:LSA_procedure_TD}
\theta_{k} = \theta_{k-1} - \alpha(\funcAw_{k} \theta_{k-1} - \funcbw_{k})\eqsp,
\end{equation}
where $\funcAw_{k}$ and $\funcbw_{k}$ are given by
\begin{equation}
\label{eq:matr_A_def}
\begin{split}
\funcAw_{k} &= \phi(s_k)\{\phi(s_k) - \gamma \phi(s'_k)\}^{\top}\eqsp, \\
\funcbw_{k} &= \phi(s_k) r(s_k,a_k)\eqsp.
\end{split}
\end{equation}
We provide the corresponding pseudocode in Algorithm~\ref{alg:TD_iid}. Under listed assumptions, we are able to check \Cref{assum:noise-level} and \Cref{assum:exp_stability}$(p)$. We first establish that \Cref{assum:noise-level} holds.

\begin{algorithm}[t!]
\caption{Temporal difference learning TD(0)}
\label{alg:TD_iid}
\SetKwInOut{Input}{Input}
\SetKwInOut{Output}{Output}
\SetKwBlock{Loop}{Loop}{end}
\SetKwBlock{Initialize}{Initialize}{end}
\SetKwFunction{Wait}{Wait}
\SetAlgoLined
\Input{feature mapping $\varphi(\cdot): \S \to \rset^{d}$, step size $\alpha$, number of iterations $n$, behavioral policy $\pi$; }
\For{$k=1,\ldots,n$}{
Receive tuple $(s_k, a_k, s'_{k})$ following \Cref{assum:generative_model}; \\
\noindent Compute update 
$\theta_{k} = \theta_{k-1} - \alpha(\funcAw_{k} \theta_{k-1} - \funcbw_{k})$
based on $\funcAw_{k}$, $\funcbw_{k}$ from \eqref{eq:matr_A_def}\;
}
\Output{tail-averaged estimate 
$\bar{\theta}_n = (2/n)\sum_{k= n/2 + 1}^{n} \theta_k$; \\ 
value function estimate $V_{\bar{\theta}_n}^{\pi}(s) =  \varphi^\top(s)\, \bar{\theta}_n$ \;
}
\end{algorithm}

\begin{lemma}
\label{prop:condition_check}
Let $\sequence{\theta}[k][\nset]$ be a sequence of TD(0) updates generated by \eqref{eq:LSA_procedure_TD} under \Cref{assum:generative_model} and \Cref{assum:feature_design}. Then this update scheme satisfies assumption \Cref{assum:noise-level} with
\begin{align}
& \label{eq:condition-check-1}\bConst{A} = 2 (1+\gamma)\eqsp, \quad \supconsteps = 2 (1+\gamma)(\norm{\thetas} + 1)\eqsp, \quad \trace{\noisecov} \leq 2 (1+\gamma)^2 (\norm{\thetas}[\covfeat]^2 + 1)\eqsp.
\end{align}
\end{lemma}
An elementary proof is given in \Cref{sec:td_proofs}. Checking  \Cref{assum:exp_stability}$(p)$ is a more delicate issue. In particular, it is crucial to determine a tight bound on the stability threshold $\alpha_{p,\infty}$. \citep{patil2023finite} contains an instance-independent bound on the maximum step size, which scales only by a factor $1-\gamma$, for the case of $2$-nd moment stability. Higher-order moments are analyzed using a modification of TD, with an additional projection. The counterpart of the exponential stability \Cref{assum:exp_stability}$(p)$ is implicitly obtained in \citep{li2023sharp}, but in this work the stability bound scales with $\lambda_{\min}$, which is unavailable in practice. To the best of our knowledge, \textbf{we provide the first instance-independent stability bound for the TD (0) algorithm beyond the $2$-nd moment}:

\begin{lemma}
\label{prop:stability_check_td}
Let $\sequence{\theta}[k][\nset]$ be a sequence of TD(0) updates generated by \eqref{eq:LSA_procedure_TD} under \Cref{assum:generative_model} and \Cref{assum:feature_design}. Then this update scheme satisfies assumption \Cref{assum:exp_stability}$(p)$ with
\begin{equation}
\label{eq:stability_threshold_td}
\textstyle 
a = (1-\gamma) \lambda_{\min} / 2\eqsp, \eqsp \conststab = 1\eqsp, \, \alpha_{p,\infty} = (1-\gamma) / (128p)\eqsp.
\end{equation}
\end{lemma}
Proof of \Cref{prop:stability_check_td} is provided in \Cref{subsec:stability_check_td}.
Note that, in strong contrast with this result, leveraging the matrix stability argument of \citep{huang2020matrix} and \citep{durmus2021tight} yield an instance-dependent stability threshold 
\begin{equation}
\label{eq:old_stability_threshold}
\textstyle 
\alpha_{p,\infty} = (1-\gamma) \lambda_{\min} / (c_0 p)
\end{equation}
for some absolute constant $c_0 > 0$. A detailed derivation of the bound \eqref{eq:old_stability_threshold} can be found in \Cref{sec:old_stability_bound}. The same order of magnitude of the step size is predicted in \cite[Theorem 1]{li2023sharp}. Thus, with the result \Cref{prop:stability_check_td}, we can prove the convergence of TD(0) for larger step sizes.
\par 
Now we are ready to adapt the conclusions of \Cref{subsec:LSA_refined} to TD learning. To represent the bounds in terms of $\norm{\cdot}[\covfeat]$ rather than the norm associated with the system matrix $\bA$, we can use the lower bound
\[
\textstyle 
\norm{\bA(\prtheta_{n} - \thetas)}^2 \geq (1-\gamma)^2 \lambda_{\min}\norm{\prtheta_{n} - \thetas}[\covfeat]^2\eqsp.
\]
The proof of this bound is provided in \Cref{lem:new_conditions_td_check}, and closely follows the idea of \cite[Lemma~5]{li2023sharp}. We begin with bounding the $2$-nd moment of the error, and immediately reformulate this result as a sample complexity bound . 

\begin{theorem}
\label{th:td_lsa_pr_2nd_moment}
Assume \Cref{assum:generative_model} and \Cref{assum:feature_design}.
Let $\sequence{\theta}[k][\nset]$ be a sequence of TD(0) updates generated by \eqref{eq:LSA_procedure_TD}.  Then for any $n \geq 2$, $\alpha \in \ocint{0;\frac{1-\gamma}{256}}$, and $\theta_0 \in \rset^d$, it holds that
\begin{align}
\label{eq:pr_theta_2nd_moment_bound_complexity}
\PE^{1/2}[ \norm{\prtheta_{n} - \thetas}[\covfeat]^2] &\lesssim  \frac{\norm{\thetas}[\covfeat] + 1}{\sqrt{\lambda_{\min} n} (1-\gamma) }\bigl(1 + \frac{\sqrt{\alpha}}{\sqrt{(1-\gamma) \lambda_{\min}}}\bigr) + \frac{\norm{\thetas}[\covfeat] + 1}{\sqrt{\alpha} (1-\gamma)^{3/2} \lambda_{\min} n} \\
&+ f_{1}(\alpha,\lambda_{\min},n) \bigl(1 - \frac{\alpha (1-\gamma) \lambda_{\min}}{2} \bigr)^{n/2} \norm{\theta_0 - \thetas}\eqsp, 
\end{align}
where $f_{1}(\alpha, \lambda_{\min}, n)$ is a polynomial function in $1/\alpha, 1/\lambda_{\min}, n$ specified in \Cref{sec:td_missing}-\eqref{eq:pr_theta_2nd_moment_bound_appendix}. 
\end{theorem}
\begin{corollary}
\label{cor:td_lsa_pr_2nd_moment}
Under the assumptions of \Cref{th:td_lsa_pr_2nd_moment}, to achieve $\PE[\norm{\prtheta_{n} - \thetas}[\covfeat]^2] \leq \varepsilon^{2}$ it is enough to use 
\begin{equation}
\tildeo\biggl(\underbrace{\frac{ \norm{\thetas}[\covfeat]^2+1}{(1-\gamma)^2\lambda_{\min}\varepsilon^2}\biggl(1+\frac{\alpha}{\lambda_{\min}(1-\gamma)}\biggr) + \term{1/\varepsilon}{1} }_{\text{variance term}} +\underbrace{\frac{1}{\alpha \lambda_{\min}(1-\gamma)} \cdot \log{\frac{\norm{\theta_0 - \thetas}}{\varepsilon}}}_{\text{initial error}} \biggr)\eqsp.
\end{equation}
TD(0) updates, where $\term{1/\varepsilon}{1} = \frac{\norm{\thetas}[\covfeat]+1}{\sqrt{\alpha}(1-\gamma)^{3/2}\lambda_{\min}\varepsilon}$.
\end{corollary}
\paragraph{Comparison to the robust SA approach.} Note that the leading term of the bound in \Cref{th:td_lsa_pr_2nd_moment} includes factors of $1/\lambda_{\min}$. This dependence is generally unavoidable if one aims to obtain the MSE bound for $\PE[\norm{\prtheta_{n} - \thetas}[\covfeat]^2]$ that scales as $1/n$. This is due to the fact that the corresponding asymptotic covariance matrix from the central limit theorem (see e.g., \citep{fort:clt:markov:2015}) for $\sqrt{n}(\prtheta_{n} - \thetas)$ could scale with $\lambda_{\min}^{-1}$. More details on the asymptotically minimax covariance bounds are provided in \Cref{sec:lower_bounds}. In contrast, within the basin of robust stochastic approximation (RSA, \citep{nemirovski2009robust}), a convergence rate for $\PE[\norm{\prtheta_{n} - \thetas}[\covfeat]^2]$ of order $\mathcal{O}(1/\sqrt{n})$ can be derived with the instance-independent choice of step size. Importantly, this rate is not affected by a worst-case factor of $\lambda_{\min}^{-1}$. This result was obtained for the TD algorithm in \citep[Theorem~2]{bhandari2018finite}.

\paragraph{Discussion and comparison.} Optimizing the bound of \Cref{cor:td_lsa_pr_2nd_moment} with respect to the step size $\alpha$ is problematic. Taking the largest possible step size $\alpha \simeq 1-\gamma$ from \eqref{eq:stability_threshold_td} yields the number of steps to reduce deterministic error of order
\[
\tildeo\bigl(\frac{1}{(1-\gamma)^2 \lambda_{\min}} \cdot \log{\frac{\norm{\theta_0 - \thetas}}{\varepsilon}}\bigr)\eqsp,
\]
which was previously reported by \citep{patil2023finite}. However, this choice of step size results in the overall sample complexity in \Cref{cor:td_lsa_pr_2nd_moment} being at least
\[
\tildeo\bigl(\frac{1}{(1-\gamma)^2 \lambda_{\min}} \cdot \log{\frac{\norm{\theta_0 - \thetas}}{\varepsilon}} + \frac{1 + \norm{\thetas}[\covfeat]^2}{(1-\gamma)^2 \lambda_{\min}^2\varepsilon^2}  \bigr)\eqsp.
\]
The $1/\varepsilon^2$ component of this bound is by a factor of $\lambda_{\min}^{-1}$ larger than the one obtained in \citep{li2023sharp}, albeit it is agrees with the bounds of \citep[Theorem~1]{patil2023finite}. The reason is that the latter paper uses instance-independent step size $\alpha \simeq (1-\gamma)$, while \citep{li2023sharp} adjusts step size with (unknown in practice) quantity $\lambda_{\min}$ as $\alpha^{(\text{small})} \simeq (1-\gamma)\lambda_{\min}$. This choice allows to improve the variance component in \Cref{cor:td_lsa_pr_2nd_moment}, but forgetting the bias would require at least
\[
\tildeo\bigl(\frac{1}{(1-\gamma)^2 \lambda_{\min}^2} \cdot \log{\frac{\norm{\theta_0 - \thetas}}{\varepsilon}}\bigr)
\]
iterations of Algorithm~\ref{alg:TD_iid}. Moreover, the remainder term $\term{1/\varepsilon}{1}$ will scale as $(1-\gamma)^{-2}\lambda_{\min}^{-3/2}$. The same phenomenon can be traced in \citep[Theorem~1]{li2023sharp}, albeit the authors do not separate the bias and variance components of the error and assume that the procedure starts at $\thetainit = 0$. This dilemma is resolved in \cite{li_accelerated_TD}, but for a variance-reduced version of TD learning algorithm.
\par 
Instantiating \Cref{th:LSA_PR_error} for TD(0), we can provide the bound on $\PE^{1/p}[ \norm{\prtheta_{n} - \thetas}[\covfeat]^p]$ for $p \geq 2$. For completeness, this result is stated in \Cref{sec:td_missing}. With Markov's inequality applied with $p = \log{(1/\delta)}$, we can translate it into the sample complexity bound. The corresponding deviation bounds for $\norm{\prtheta_{n} - \thetas}[\covfeat]$ are provided in appendix, see \Cref{sec:td_missing}-\Cref{cor:td_pr_pth_moment}.
\begin{theorem}
\label{cor:sample_complexity_td}
Fix $\varepsilon > 0$, $\delta > 0$, assume \Cref{assum:generative_model} and \Cref{assum:feature_design}.
Let $\sequence{\theta}[k][\nset]$ be a sequence of TD(0) updates generated by \eqref{eq:LSA_procedure_TD}.   
Then for any $n \geq 2$, and step size
\[
\alpha \in \bigl(0; \frac{1-\gamma}{128 \log{(n/\delta)}}\bigr]
\]
\!\!to achieve error $\norm{(\prtheta_{n} - \thetas)}[\covfeat] \leq \varepsilon$ with probability at least $1-\delta$ it takes
\begin{equation}
\label{eq:td_hpd_bound}
\tildeo\biggl(\frac{(\norm{\thetas}[\covfeat]^2 + 1)\log{(1/\delta)}}{(1-\gamma)^2 \lambda_{\min} \varepsilon^{2}} \left(1 + \frac{\alpha \log{(1/\delta)}}{(1-\gamma)\lambda_{\min}} \right) + \term{1/\varepsilon,\delta}{2} + \frac{\log{\frac{\norm{\theta_0 - \thetas}}{\varepsilon}}}{\alpha \lambda_{\min}(1-\gamma)} \biggr)
\end{equation}
TD(0) updates, where $\term{1/\varepsilon,\delta}{2} = \frac{(\norm{\thetas}[\covfeat]+1)\log{(1/\delta)}}{\sqrt{\alpha}(1-\gamma)^{3/2}\lambda_{\min}\varepsilon}$.
\end{theorem}

\paragraph{Discussion and comparison.} Note that in \Cref{cor:sample_complexity_td} the symbol $\tildeo$ hides logarithmic dependencies in $\lambda_{\min}, 1-\gamma$, and $n$, but not in $1/\delta$. Again the direct optimization of the bound \Cref{cor:sample_complexity_td} w.r.t. $\alpha$ yield to the same dilemma as in case of $2$-nd moment. The stochastic part of the complexity bound \eqref{eq:td_hpd_bound} scales inversely proportional to $\lambda_{\min}^2$, which is worse than the scaling of the deterministic component of the error. At the same time, choosing the smaller step size
\begin{equation}
\label{eq:step_size_bound_instance-dependent}
\alpha = \frac{(1-\gamma) \lambda_{\min}(\covfeat)}{128 (p + \log{n})}\eqsp,
\end{equation}
we retrieve the leading variance term of deviation bound \citep[Theorem~1]{li2023sharp}, while improving the second-order term in $\lambda_{\min}$. Indeed, \citep[Theorem~1]{li2023sharp} yields the high-probability bound of order 

\[
\tildeo\biggl(\frac{(\norm{\thetas}[\covfeat]^2 + 1)\log{(d/\delta)}}{(1-\gamma)^2 \lambda_{\min} \varepsilon^{2}} + \frac{(1+\norm{\thetas}[\covfeat])\log{(nd/\delta)}}{(1-\gamma)^2\lambda_{\min}^2 \varepsilon}\biggr)
\]
in order to achieve {\small$\norm{\prtheta_{n} - \thetas}[\covfeat] \leq \varepsilon$} with probability at least {\small$1 - \delta$}. This result is achieved for the step size $\alpha$ which scales similarly to \eqref{eq:step_size_bound_instance-dependent}. Also, compared to \citep{li2023sharp}, we obtain a clear separation between the deterministic and stochastic parts of the error, and remove the explicit dependence upon the feature dimension $d$. Note, however, that the dependence upon $d$ is hidden implicitly inside $\lambda_{\min}$.

%% file: td_lower.tex
In this section we present a version of \Cref{th:td_lsa_pr_2nd_moment} with a leading variance term consistent with the minimax lower bound due to \citep[Proposition~1]{li_accelerated_TD}.  We first write the TD(0) noise covariance matrix 
\[
\noisecovtd = \PE[ \bigl((\phi(s_k) - \gamma \phi(s'_k))^{\top}\thetas - r_k\bigr)^{2} \phi(s_k) \phi(s_k)^{\top}]\eqsp,
\]
which corresponds to the general LSA noise covariance matrix $\noisecov$ defined in \eqref{eq:def_noise_cov}. We also define the transformed covariance matrix 
\[
\noisecovopt = \covfeat^{1/2} \bA^{-1} \noisecovtd \bA^{-T} \covfeat^{1/2}\eqsp,
\]
which corresponds to the covariance of modified noise variables $\covfeat^{1/2} \bA^{-1}\funnoisew$. Now let us introduce the counterpart of \Cref{th:td_lsa_pr_2nd_moment} with the modified leading (w.r.t. the sample size $n$) term. 
\begin{theorem}
\label{th:leading_term_variance}
Assume \Cref{assum:generative_model} and \Cref{assum:feature_design}.
Let $\sequence{\theta}[k][\nset]$ be a sequence of TD(0) updates generated by \eqref{eq:LSA_procedure_TD}.   
Then for any $p \geq 2$, $n \geq 2$, $\alpha \in \bigl(0; \frac{1-\gamma}{256}\bigr]$, it holds that 
\begin{equation}
\label{eq:modified_pr_upper_bound}
\begin{split}
\PE^{1/2}[\norm{\prtheta_{n} - \thetas}[\covfeat]^2] &\lesssim \frac{\sqrt{\trace{\noisecovopt}}}{n^{1/2}} + \underbrace{\frac{1+\norm{\thetas}[\covfeat]}{(1-\gamma)^{3/2} \lambda_{\min} n^{1/2}} \left(\frac{1}{\sqrt{\alpha n}} + \sqrt{\alpha} \right)}_{\term{n,\lambda_{\min}}{3}} \\
&+ f_{2}(\alpha, \lambda_{\min}, n) \bigl(1- \alpha (1-\gamma) \lambda_{\min}\bigr)^{n/2} \norm{\theta_0 - \thetas} \eqsp,
\end{split}
\end{equation}
where $f_{2}(\alpha, \lambda_{\min}, n)$ is a polynomial in $1/\alpha, 1/\lambda_{\min}, n$ specified in \Cref{sec:app:lower_bounds}-\eqref{eq:pr_theta_2_moment_bound_new}.
\end{theorem}

The proof is postponed to \Cref{sec:app:lower_bounds}, along with the analogous $p$-th moment bound. We highlight the fact that the leading term of \eqref{eq:modified_pr_upper_bound} scales with the quantity $\trace{\noisecovopt}$ corresponding to the instance optimal variance given in \citep[Section~2]{li_accelerated_TD} and \citep{mou2020linear}. At the same time, with simple algebraic manipulations one can prove an upper bound 
\[
\textstyle
\trace{\noisecovopt} \leq \frac{\norm{\thetas}[\covfeat]^2 + 1}{(1-\gamma)^2\lambda_{\min}}\eqsp,
\]
thus recovering the result obtained in \Cref{th:td_lsa_pr_2nd_moment} before. However, the bound of \eqref{eq:modified_pr_upper_bound} contains also a term $\term{n,\lambda_{\min}}{3}$, which scales directly with $\lambda_{\min}^{-1}$. Moreover, even setting $\alpha$ as $n^{-\varkappa}$, $\varkappa \in (0,1)$ would require to use large sample size $n$ in order that the optimal noise term $(\trace{\noisecovopt}/n)^{1/2}$ starts to dominate. This is on par with empirical evaluation from \cite{khamaru2021temporal}. Now we reformulate \Cref{th:leading_term_variance} as a sample complexity bound.

\begin{corollary}
\label{cor:lower_bound_seond_moment}
Under the assumptions of \Cref{th:leading_term_variance}, to achieve the weighted MSE $\PE[\norm{\prtheta_{n} - \thetas}[\covfeat]^2] \leq \varepsilon^{2}$ requires 
\begin{equation}
\tildeo\bigl(\underbrace{\frac{\log{\frac{\norm{\theta_0 - \thetas}}{\varepsilon}}}{\alpha \lambda_{\min}(1-\gamma)}}_{\text{initial error}} + \underbrace{\frac{\trace{\noisecovopt}}{\varepsilon^{2}} + \frac{\alpha (1 + \norm{\thetas}[\covfeat]^2)}{(1-\gamma)^3 \lambda_{\min}^2\varepsilon^2} + \term{1/\varepsilon}{4}}_{\text{variance term}}\bigr)\eqsp,
\end{equation}
TD(0) updates, where $\term{1/\varepsilon}{4}$ scales linearly with $1/\varepsilon$. 
\end{corollary}

%% file: markov_main.tex
\!Here we present an extension of the results of \Cref{sec:td_learning} under Markovian sampling. The corresponding results generalize the high probability bounds of \Cref{cor:td_pr_pth_moment} and \Cref{cor:sample_complexity_td}. We start with the following assumption:

\begin{assumTD}
\label{assum:generative_model_markov}
Training tuples $(s_k,a_k,s_{k+1})$ are generated sequentially following the generative model $a_{k} \sim \pi(\cdot|s_{k})$, $s_{k+1} \sim \PMDP(\cdot|s_k,a_k)$.
\end{assumTD}

Note that the assumption \Cref{assum:generative_model_markov} yields that the sequence $\{s_k\}_{k \in \nset}$ is a Markov chain with the Markov kernel $\PMDP_{\pi}(\cdot|s)$ defined in \eqref{eq:transition_matrix_P_pi}, that corresponds to a classical problem of on-policy evaluation. However, since we are using a single chain for policy evaluation, our subsequent analysis requires to impose ergodicity constraints on $\PMDP_{\pi}(\cdot|s)$.
\begin{assumTD}
\label{assum:P_pi_ergodicity}
The Markov kernel $\PMDP_{\pi}$ admits a unique invariant distribution $\mu$ and is uniformly geometrically ergodic, that is, there exist $\taumix \in \nset$, such that for any $k \in \nset$, it holds that
\begin{equation}
\label{eq:drift-condition-main}
 \sup_{s,s' \in \S} (1/2) \norm{\PMDP_{\pi}^{k}(\cdot|s) - \PMDP_{\pi}^{k}(\cdot|s')}[\mathsf{TV}] \leq  (1/4)^{\lfloor k / \taumix \rfloor}\eqsp.
\end{equation}
\end{assumTD}
We note that \Cref{assum:P_pi_ergodicity} is widely used in theoretical RL and stochastic optimization, see, e.g. \citep{bhandari2018finite,nagaraj2020least,dorfman2022adapting,patil2023finite}. 
The parameter $\taumix$  is the \emph{mixing time}, see e.g. \citep{paulin_concentration_spectral}. The constant $1/4$ in \eqref{eq:drift-condition-main} can be changed to arbitrary constant in $\coint{0,1}$ with proper rescaling of $\taumix$.
\begin{algorithm}[t!]
\caption{Temporal difference learning TD(0) with data drop}
\label{alg:TD_skip}
\SetKwInOut{Input}{Input}
\SetKwInOut{Output}{Output}
\SetKwBlock{Loop}{Loop}{end}
\SetKwBlock{Initialize}{Initialize}{end}
\SetKwFunction{Wait}{Wait}
\SetKwFunction{ClientMain}{ClientLocalTraining}
\SetAlgoLined
\Input{features $\varphi(\cdot): \S \to \rset^{d}$, step size $\alpha$, number of iterations $n$, burn-in size $n_0$, behavioral policy $\pi$, time window $q \in \nsets$}
Compute number of blocks $m = \lfloor n / q \rfloor$ \;\\
\For{$k=1,\ldots,n$}{
Receive tuple $(s_k, a_k, s_{k+1})$ following \Cref{assum:P_pi_ergodicity} \;\\
\uIf{$k = q j, j \in \nset$}{
    Compute update 
    $
    \tilde{\theta}_{j} = \tilde{\theta}_{j-1} - \alpha(\funcAw_{k} \tilde{\theta}_{j-1} - \funcbw_{k})
    $
    based on $\funcAw_{k}$, $\funcbw_{k}$ from \eqref{eq:matr_A_def}\;
  }
\Else{
    skip current learning tuple\;
}
}
\Output{tail-averaged estimate 
$
\bar{\theta}_{n} = (2/m) \sum_{k=m/2+1}^{m} \tilde{\theta}_k
$ \;\\ 
value function estimate $V_{\bar{\theta}_n}^{\pi}(s) =  \varphi^\top(s)\, \bar{\theta}_n$ \;
}
\end{algorithm}

The algorithm that we analyze in the Markovian setting is not a standard version of TD(0), but its modification with data-drop. It is summarized in \Cref{alg:TD_skip} and has additional parameter $q \in \nset$. We take every $q$-th tuple from the trajectory $\{(s_{k},a_{k},s_{k+1})\}_{k \in \nset}$. Parameter $q$ here needs to be properly adjusted with $\taumix$, see \Cref{cor:td_pr_deviation_markov} below. The data-drop approach was previously explored in \citep{nagaraj2020least} for the general least-squares problems. The authors of \citep{nagaraj2020least} further established that this strategy is optimal in a sense that required number of samples scales linearly with  \(\operatorname{t_{mix}}\), and this dependence is worst-case optimal. In the context of TD(0) algorithm the same approach was suggested and studied by \citep{patil2023finite}, with the restriction to finite state space setting. Now we are ready to state and prove the counterpart of \Cref{cor:sample_complexity_td} for the case of TD(0) updates generated by Algorithm~\ref{alg:TD_skip}.

\begin{theorem}
\label{cor:td_pr_deviation_markov}
Assume \Cref{assum:feature_design}, \Cref{assum:generative_model_markov}, and \Cref{assum:P_pi_ergodicity}. Fix $\delta \in (0,1/3)$ and let $\bar{\theta}_{n}$ be a tail-averaged estimate generated by Algorithm~\ref{alg:TD_skip} run with parameters
\[
\alpha = \frac{1-\gamma}{128\log{(n/\delta)}}\eqsp, \quad q = \left\lceil \frac{\taumix \log{(n/\delta)}}{\log{4}}\right\rceil\eqsp,
\]
given that the sample size $n$ satisfies $n \geq \frac{\log{(1/\delta)}}{(1-\gamma)^2} \vee \frac{2 \taumix \log(4/\delta)}{\log{4}}$. Then it holds with probability at least $1 - 3 \delta$ that 
\begin{multline}
\label{eq:deviation_bound_pth_moment_markov}
\norm{\prtheta_{n} - \thetas}[\covfeat] 
\lesssim \frac{(\norm{\thetas}[\covfeat] + 1) \taumix^{1/2} \log{(n/\delta)}}{n^{1/2} (1-\gamma) \lambda_{\min}} \\
+ \exp\left\{-\frac{(1-\gamma)^2 \lambda_{\min}n}{128 \taumix  \log^2(n/\delta)}\right\} \frac{\norm{\theta_0 - \thetas} \taumix \log^{2}{(n/\delta)}}{(1-\gamma)^2 \lambda_{\min} n}\eqsp.
\end{multline}
\end{theorem}
The proof is postponed to \Cref{sec:rosenthal_markov} and is based on Berbee's coupling lemma, see \citep{berbee1979}. Note that the result of \Cref{cor:td_pr_deviation_markov} is slightly suboptimal compared to \Cref{cor:td_pr_pth_moment}. Indeed, the leading term with respect to $n$ of the bound \eqref{eq:deviation_bound_pth_moment_markov} scales with $\log{(1/\delta)}$ instead of $\sqrt{\log{(1/\delta)}}$ in the \iid\ counterpart. That is, the leading term of \eqref{eq:deviation_bound_pth_moment_markov} exhibits subexponential behaviour instead of sub-Gaussian. This behaviour is a result of using Berbee's coupling lemma.

\paragraph{Discussion} The practical application of data-drop approach is limited, since the gap size $q$ in \Cref{alg:TD_skip} should scale with unknown in practice parameter \(\operatorname{t_{mix}}\). It is possible to analyze under \Cref{assum:generative_model_markov} and \Cref{assum:P_pi_ergodicity} a direct counterpart of \Cref{alg:TD_iid} without data-drop. The key difficulty of such analysis is to verify an exponential stability assumption \Cref{assum:exp_stability}. This is done, for example, in \citep[Proposition~7]{durmus2022finite} for the general LSA problem. However, the respective stability threshold $\alpha_{p,\infty}$ scales as $\taumix^{-1}$. This means, that from theoretical perspective we still observe the following dilemma - either we run data-drop algorithm with the number of dropped observations, which scales with $\taumix$, or we run \Cref{alg:TD_iid} without data-drop, but the step size $\alpha$ has to be adjusted with $\taumix^{-1}$.
\par 
Similarly to the \iid\ setting, we can rewrite \Cref{cor:td_pr_deviation_markov} as a sample complexity bound. 

\begin{corollary}
\label{cor:sample_complexity_td_markov}
Under assumptions of \Cref{cor:td_pr_deviation_markov} in order to achieve $\norm{\prtheta_{n} - \thetas}[\covfeat] \leq \varepsilon$ with probability at least $1 - 3\delta$ it requires
\[
\mathcal{\tilde{O}}\left(\frac{\taumix (\norm{\thetas}[\covfeat]^2 + 1) \log{(1/\delta)}}{(1-\gamma)^2 \lambda_{\min}^{2} \varepsilon^{2}} + \frac{\taumix \log^2{(1/\delta)}}{\lambda_{\min}(1-\gamma)^2}\log{\frac{\norm{\theta_0 - \thetas}}{\varepsilon}}\right)
\]
observation used in Algorithm~\ref{alg:TD_skip}.
\end{corollary}
Note that in \Cref{cor:sample_complexity_td_markov} the symbol $\tildeo$ hides logarithmic dependencies in $\lambda_{\min}, 1-\gamma$, and $n$, but not in $1/\delta$. The sample complexity bounds of \Cref{cor:sample_complexity_td_markov} matches the ones coming from \Cref{cor:sample_complexity_td} up to an additional $\taumix$ factor and extra factor of $\sqrt{\log{(1/\delta)}}$. We believe that a factor of $\sqrt{\log{(1/\delta)}}$ can be removed using the analysis based on \Cref{alg:TD_iid} without data drop and appropriate versions of Rosenthal inequality for Markov chains and leave it as a direction for further work. 

%% file: conclusion.tex
In this paper we presented a refined analysis of linear stochastic approximation algorithms and provide high-probability and sample complexity bounds for the TD(0) algorithm via the exponential stability argument. Our approach allows to obtain high-probability bounds without requiring projections or instance-dependent step size. Further research directions include generalizing the high-order error bounds to the Markov setting for versions of the TD learning algorithm that do not use the data drop modification, while maintaining the precise variance from the corresponding central limit theorem. Second, our version of Algorithm~\ref{alg:TD_skip} requires knowledge of $\taumix$, which is a common drawback shared by the versions of SGD with data drop algorithm \citep{nagaraj2020least}. To the best of our knowledge, it is an open problem to develop a version of this algorithm which would be oblivious to $\taumix$.

%% file: supplementary-material.tex
\section{Proofs of LSA error bounds presented in \Cref{sec:lsa}}
\label{sec:general_lsa_proofs}
Recall that we consider a sequence of LSA estimates $\{\theta_n\}_{n \in \nset}$ given by the recurrence
\begin{equation}
\label{eq:LSA_refined}
\theta_{n} = \theta_{n-1} - \alpha \{ \funcAw_n \theta_{n-1} - \funcbw_n \} \eqsp,~~ n \geq 1 \eqsp.
\end{equation}
In the formula above we use $\funcAw_n$ and $\funcbw_n$ as a shorthand notations for $\funcAw(\State_n)$ and $\funcbw(\State_n)$, respectively. We use the same convention throughout the appendix section. Our analysis relies heavily on the stability assumption \Cref{assum:exp_stability} for the matrix products of the form
\[
\ProdBa_{1:n} = \prod_{i=1}^{n}(\Id - \alpha \funcAw_{i})\eqsp.
\]
We obtain the following refined bound on the last iterate error of the procedure \eqref{eq:LSA_refined}:
\begin{theorem}
\label{th:LSA_last_iterate_refined}
\begin{enumerate}[(i)]
\item Assume \Cref{assum:noise-level} and \Cref{assum:exp_stability}($2$). Then, for any $\alpha \in (0;\alpha_{2,\infty})$ and $n \in \nset$, it holds that
\begin{equation}
\label{eq:LSA_last_refined_2nd_moment}
\PE^{1/2}[\norm{\theta_n - \thetas}^2] \leq \conststab[2] (1- \alpha a)^{n} \norm{\theta_0 - \thetas} + \frac{\conststab[2] \sqrt{\alpha \trace{\noisecov}}}{\sqrt{a}}\eqsp.
\end{equation}
\item  Let $p \geq 2$. Assume \Cref{assum:noise-level} and \Cref{assum:exp_stability}($p$). Then, for any $\alpha \in (0;\alpha_{p,\infty})$ and $n \in \nset$, it holds that

\begin{equation}
\label{eq:LSA_last_refined}
\PE^{1/p}[\norm{\theta_n - \thetas}^p] \leq \conststab[p] (1- \alpha a)^{n} \norm{\theta_0 - \thetas} + \frac{\conststab[p] p \sqrt{\alpha}}{\sqrt{a}} \supconsteps \eqsp.
\end{equation}
\item Grant \Cref{assum:noise-level} and \Cref{assum:exp_stability}($\ell$) for any $\ell \geq 2$. Then for any $p \geq 2$,  $n \geq 2$ and  $\alpha \in \coint{0;\alpha_{p+\log{n},\infty}}$, it holds that
\begin{equation}
\label{eq:Rosenthal_LSA}
\begin{split}
\PE^{1/p}[\norm{\theta_n - \thetas}^p] &\leq \conststab[p+\log{n}] (1- \alpha a)^{n} \norm{\theta_0 - \thetas} + \bConst{\mathsf{Rm}, 1} p^{1/2} \frac{\conststab[p+\log{n}] \sqrt{\alpha \trace{\noisecov}}}{\sqrt{a}} \\
&\qquad +\bConst{\mathsf{Rm}, 2} \rme \alpha p \conststab[p+\log{n}] \supconsteps\eqsp,
\end{split}
\end{equation}
where $\bConst{\mathsf{Rm}, 1} = 60$ and $\bConst{\mathsf{Rm}, 2} = 60\rme$ are constants from the martingale version of Rosenthal's inequality \cite[Theorem~4.1]{pinelis_1994}.
\end{enumerate}
\end{theorem}
\begin{proof}
Using the error expansion technique from \cite{aguech2000perturbation} (see also \cite{durmus2022finite}), we decompose $\theta_n$ into a transient and fluctuation terms
\[
\theta_{n} - \thetas = \utheta_{n} + \vtheta_{n}\eqsp,
\]
where we have defined the quantities
\begin{equation}
\label{eq:LSA_recursion_expanded}
\textstyle
\utheta_{n} = \ProdBa_{1:n} \{ \theta_0 - \thetas \} \eqsp, \quad \vtheta_{n} = - \alpha \sum_{j=1}^n \ProdBa_{j+1:n} \funnoisew_j\eqsp.
\end{equation}
The first term $\utheta_{n}$ in the error decomposition \eqref{eq:LSA_recursion_expanded} is transient and reflects the forgetting of the initial error of the LSA. It can be directly controlled using the assumption \Cref{assum:exp_stability}($p$), $p \geq 2$:
\[
\PE^{1/p}[\norm{\ProdBa_{1:n} \{ \theta_0 - \thetas \}}^{p}] \leq \conststab[p] (1- \alpha a)^{n} \norm{\theta_0 - \thetas}\eqsp.
\]
In order to control the fluctuation term $\vtheta_{n}$, we note that it is a reverse martingale w.r.t. filtration $\mathcal{F}_k = \sigma(Z_{j}, j \geq k)$. Thus, applying the Burkholder inequality \cite[Theorem~8.6]{osekowski:2012}, we obtain that, assuming \Cref{assum:exp_stability}($p$)
\begin{align}
\PE^{1/p}\left[ \left\Vert\alpha \sum\nolimits_{j=1}^n \ProdBa_{j+1:n} \funnoisew_j \right\Vert^{p}\right]
& \leq \alpha p \left(\PE^{2/p}\left[\left(\sum\nolimits_{j=1}^{n}\normop{\ProdBa_{j+1:n} \funnoisew_j}^{2}\right)^{p/2}\right]\right)^{1/2} \\
& \leq  \alpha p \left(\sum\nolimits_{j=1}^{n}\PE^{2/p}\bigl[\normop{\ProdBa_{j+1:n} \funnoisew_j}^{p} \bigr]\right)^{1/2} \\
&\leq \alpha p \conststab[p] \bigl( \PE^{2/p}\bigl[\normop{\funnoisew_1}^{p}] \sum\nolimits_{j=1}^{n} (1- \alpha a)^{2(n-j)}\bigr)^{1/2} \\
&\leq \frac{\conststab[p] p \sqrt{\alpha}}{\sqrt{a}} \supconsteps\eqsp,
\end{align}
where for the last bound we additionally used that $\alpha a \leq 1/2$. Substituting the bounds above into \eqref{eq:LSA_recursion_expanded} completed the proof. Obtaining the second moment bound \eqref{eq:LSA_last_refined_2nd_moment} follows the same lines as above using the martingale structure of $\vtheta_{n}$, that is,
\begin{align}
\PE^{1/2}\biggl[\bigg\Vert\alpha \sum_{j=1}^n \ProdBa_{j+1:n} \funnoisew_j \bigg\Vert^{2}\biggr]
&\leq \alpha \biggl(\sum_{j=1}^n \PE\left[\norm{\ProdBa_{j+1:n} \funnoisew_j}^2\right]\biggr)^{1/2} \\
&\leq \alpha \conststab[p+\log{n}] \biggl(\sum_{j=1}^{n} (1 - \alpha a)^{2(n-j)}\trace{\noisecov}\biggr)^{1/2} \\
&\leq \frac{\conststab[p+\log{n}] \sqrt{\alpha}}{\sqrt{a}} \{\trace{\noisecov}\}^{1/2} \eqsp.
\end{align}
Now we aim to obtain the refined bound \eqref{eq:Rosenthal_LSA}. For $k \in \{1,\ldots,n\}$, we set $\mcf_k= \sigma(\State_s \, : \,  s \leq k)$, and $\mcf_{0} = \{\emptyset, \msz\}$. Then it is easy to see that $\CPE{\ProdBa_{j+1:n} \funnoisew_j}{\mcf_{j-1}}= 0$ for any $j = 1,\ldots,n$. Hence, applying the Pinelis version of Rosenthal inequality \cite[Theorem~4.1]{pinelis_1994}, we obtain that
\begin{multline}
\label{eq:rosenthal_pinelis_lsa}
\PE^{1/p}\left[\left\Vert\alpha \sum\nolimits_{j=1}^n \ProdBa_{j+1:n} \funnoisew_j \right\Vert^{p}\right]
\leq \alpha \bConst{\mathsf{Rm}, 1} p^{1/2} \PE^{1/p}\left[\left(\sum\nolimits_{j=1}^{n}\CPE{\normop{\ProdBa_{j+1:n} \funnoisew_j}^2}{\mcf_{j-1}}\right)^{p/2}\right] \\ + \alpha p \bConst{\mathsf{Rm}, 2}\, \PE^{1/p}\left[\max_{j} \normop{\ProdBa_{j+1:n} \funnoisew_j}^p\right]\eqsp.
\end{multline}
Since $\funnoisew_j$ is independent of $\ProdBa_{j+1:n}$, it is easy to see that
\begin{align}
\CPE{\normop{\ProdBa_{j+1:n} \funnoisew_j}^2}{\mcf_{j-1}}
&= \CPE{\CPE{\normop{\ProdBa_{j+1:n} \funnoisew_j}^2}{\mcf_{j}}}{\mcf_{j-1}}
\leq \conststab[p+\log{n}]^2 (1-\alpha a)^{2(n-j)} \CPE{\norm{\funnoisew_j}^2}{\mcf_{j-1}} \\
&= \conststab[p+\log{n}]^2 (1-\alpha a)^{2(n-j)} \trace{\noisecov}\eqsp.
\end{align}
Thus, with simple algebra and using that $\alpha a \leq 1/2$, we get that
\[
\alpha \bConst{\mathsf{Rm}, 1} p^{1/2} \PE^{1/p}\biggl[\biggl(\sum_{j=1}^{n}\CPE{\normop{\ProdBa_{j+1:n} \funnoisew_j}^2}{\mcf_{j-1}}\biggr)^{p/2}\biggr]
\leq \bConst{\mathsf{Rm}, 1} p^{1/2} \frac{\conststab[p+\log{n}] \sqrt{\alpha \trace{\noisecov}}}{\sqrt{a}}\eqsp.
\]
In order to control the remainder term in Rosenthal's inequality \eqref{eq:rosenthal_pinelis_lsa}, we note that, with $q = p + \log{n}$, it holds
\begin{equation}
\begin{split}
\PE^{1/p}\bigl[\max_{j} \normop{\ProdBa_{j+1:n} \funnoisew_j}^p\bigr] 
&\leq \PE^{1/q}\bigl[\max_{j} \normop{\ProdBa_{j+1:n} \funnoisew_j}^q\bigr] \leq \left(\sum\nolimits_{j=1}^{n} \PE[\normop{\ProdBa_{j+1:n} \funnoisew_j}^q] \right)^{1/q} \\
&\leq \conststab[p+\log{n}] n^{1/q} \supconsteps \leq \rme \conststab[p+\log{n}] \supconsteps\eqsp.
\end{split}
\end{equation}
Now it remains to combine the bounds above in \eqref{eq:rosenthal_pinelis_lsa}, and the result of \eqref{eq:Rosenthal_LSA} follows.
\end{proof}

Note that \Cref{th:LSA_last_iterate_refined} provides $2$ bounds for the last LSA iterate error, \eqref{eq:LSA_last_refined} and \eqref{eq:Rosenthal_LSA}. The second one might provide an improvement, since $\supconsteps \geq \sqrt{\trace{\noisecov}}$. If we aim to obtain bounds in terms of solely the noise variance $\trace{\noisecov}$, we need that the reverse inequality holds, that is, 
\[
\supconsteps \leq c_{1} \sqrt{\trace{\noisecov}}
\]
for some appropriate constant $c_1 > 0$. The problem is that the scaling of $c_1$ with instance-dependent quantities of \Cref{sec:td_learning} might be pessimistic. That is why it is desirable to have this dependence coming with additional $\alpha$ factor, instead of just $\sqrt{\alpha}$ in \eqref{eq:LSA_last_refined}.
\par
Now we state and proof the similar results for the Polyak-Ruppert averaged estimator $\prtheta_{n_0,n}$. We use the following decomposition based on the summation by parts formula:
\begin{align}
\label{eq:pr_err_decompose}
\bA\left(\prtheta_{n_0,n} -\thetalim\right) & =  \frac{\theta_{n_0}-\theta_{n}}{\alpha (n - n_0)} - \frac{\sum_{t=n_0}^{n-1} e\left(\theta_{t}, Z_{t+1} \right)}{n-n_0} \eqsp,
\end{align}
where we have defined
\begin{equation}
\label{eq:pr_e_definition}
e(\theta,z) = \zmfuncA{z} \theta- \zmfuncb{z} = \funcnoise{z} + \zmfuncA{z} (\theta - \thetas) \eqsp.
\end{equation}
The decomposition above is nothing but summation by parts formula used in \cite{mou2020linear}, yet it can be traced to the preceding papers. Recall also that we have set the notation $\prtheta_{n}$ as an alias for $\prtheta_{n_0,n}$ used with $n_0 = n/2$. Before we proceed to the proof of \Cref{th:LSA_PR_error}, we first provide a simpler statement regarding the $2$-nd moment of the PR-averaged error.

\begin{theorem}
\label{th:lsa_pr_2nd_moment_appendix}
Assume \Cref{assum:noise-level} and \Cref{assum:exp_stability}($2$). Then for any $n \geq 2$, $\alpha \in (0;\alpha_{2,\infty}]$, it holds that
\begin{multline}
\label{eq:pr_theta_2nd_moment_bound_appendix}
\PE[\norm{\bA(\prtheta_{n} - \thetas)}^2] \lesssim \conststab[2]^2 (1 - \alpha a)^{n} \left(\frac{1}{\alpha^2n^2} + \frac{\bConst{A}^2}{\alpha a n^2}\right)\norm{\theta_0 - \thetas}^2 \\
+\frac{\trace{\noisecov}}{n}\left(1 + \frac{\conststab[2]^2 \bConst{A}^2 \alpha}{a}\right) + \frac{\conststab[2]^2 \trace{\noisecov}}{\alpha a n^2}\eqsp.
\end{multline}
\end{theorem}
\begin{proof}
Our proof is essentially a version of \cite[Proposition~5]{durmus2022finite} with tighter instance-dependent bound on the last LSA iterate error provided by \Cref{th:LSA_last_iterate_refined}. We leverage the error decomposition \eqref{eq:pr_err_decompose}. Then we get
\[
\PE[\norm{\bA(\prtheta_{n} - \thetas)}^2] \leq \underbrace{\frac{8\PE[\norm{\theta_{n/2} - \theta_n}^2]}{\alpha^2 n^2}}_{T_1} \, + \, \underbrace{\frac{8 \PE[\norm{\sum_{t=n/2}^{n-1} e\left(\theta_{t}, Z_{t+1} \right)}^2]}{n^2}}_{T_2}\eqsp,
\]
and estimate the terms $T_1$ and $T_2$ separately. Applying the bounds of \Cref{th:LSA_last_iterate_refined}, we get first that
\[
T_1 \lesssim \frac{\conststab[2]^2 (1 - \alpha a)^{n}\norm{\theta_0 - \thetas}^2}{\alpha^2 n^2} + \frac{\conststab[2]^2 \trace{\noisecov}}{\alpha a n^2}\eqsp.
\]
Similarly, since $e\left(\theta_{t}, Z_{t+1} \right)$ is a martingale-difference sequence w.r.t. filtration $\mathcal{F}_k = \sigma(Z_{j}, j \leq k)$, we get the following bound for $T_2$:
\[
T_2 \leq n^{-2}\sum_{t=n/2}^{n-1}\PE[\norm{e\left(\theta_{t}, Z_{t+1} \right)}^2] \lesssim \frac{\trace{\noisecov}}{n} + \frac{\conststab[2]^2 \bConst{A}^2 (1- \alpha a)^{n}\norm{\theta_0 - \thetas}^2}{\alpha a n^2} + \frac{\conststab[2]^2 \bConst{A}^2 \alpha \trace{\noisecov}}{an}\eqsp,
\]
and it remains to combine the above bounds.
\end{proof}

Now we are ready to proceed with the main result of this section, that is, with the $p$-moment error bound \Cref{th:LSA_PR_error}.
\begin{theorem}
\label{th:lsa_pr_error_p_moment}
Assume \Cref{assum:noise-level} and \Cref{assum:exp_stability}($\infty$). Then for any $p \geq 2$, $n \geq 2$, $\alpha \in \coint{0;\alpha_{p+\log{n},\infty}}$, it holds that
\begin{equation}
\label{eq:pr_theta_p_moment_bound}
\begin{split}
&\PE^{1/p}[\norm{\bA(\prtheta_{n} - \thetas)}^p]
\lesssim \frac{p^{1/2}\sqrt{\trace{\noisecov}}}{n^{1/2}}\left(1 + \frac{ \conststab[p+\log{n}] \sqrt{\alpha p} \bConst{A}}{\sqrt{a}} + \frac{\conststab[p+\log{n}] \bConst{A} \alpha p \supconsteps}{\sqrt{\trace{\noisecov}}} \right)  \\
& \qquad +\frac{\conststab[p+\log{n}] p \supconsteps}{n}\left(1 + \bConst{A} \alpha (p + \log n)\right) \\
& \qquad + \frac{\conststab[p+\log{n}] p^{1/2} \sqrt{\trace{\noisecov}}}{\sqrt{a} n} \biggl[ \frac{1}{\sqrt{\alpha}} + p^{1/2} \bConst{A} \sqrt{\alpha (p + \log n)}\biggr] \\
& \qquad +\conststab[p+\log{n}] (1-\alpha a)^{n/2} \left(\frac{1}{\alpha n} + \frac{p \bConst{A}}{\sqrt{\alpha a} n}\right)\norm{\theta_0 - \thetas}\eqsp.
\end{split}
\end{equation}
\end{theorem}
\begin{proof}
The proof is also based on the expansion formula \eqref{eq:pr_err_decompose}. We recall that we set $n_0 = n/2$. Then, with the direct application of Minkowski's inequality, we obtain
\begin{align}
\label{eq:bound_p_moment_iid_T_1_T_2_decomp}
\PE^{1/p}\left[\norm{\bA\left(\prtheta_{n}-\thetas\right)}^{p}\right] \leq \underbrace{\frac{\PE^{1/p}[\norm{\theta_{n/2}-\theta_{n}}^p]}{\alpha n}}_{T_1} + \underbrace{\frac{\PE^{1/p}\bigl[\norm{\sum\nolimits_{t=n/2}^{n-1}\rme(\theta_{t},\State_{t+1})}^{p}\bigr]}{n}}_{T_2}\eqsp,
\end{align}
and bound $T_1$, $T_2$ separately. Note that $T_1$ is a remainder term (w.r.t. sample size $n$), and thus we can control it using a simple bound on the last iterate error provided in \Cref{th:LSA_last_iterate_refined}-\eqref{eq:Rosenthal_LSA}. Proceeding this way, we obtain
\[
T_1 \lesssim \frac{\conststab[p+\log{n}] (1- \alpha a)^{n/2} \norm{\theta_0 - \thetas}}{\alpha n} + \frac{\conststab[p+\log{n}] p^{1/2} \sqrt{\trace{\noisecov}}}{\sqrt{\alpha a} n} + \frac{p \conststab[p+\log{n}] \supconsteps}{n}\eqsp.
\]
Now we proceed with bounding $T_2$. Using again Minkowski's inequality, we get
\[
\textstyle
T_2 \leq n^{-1}\, \PE^{1/p}\bigl[\norm{\sum_{t=n/2}^{n-1}\funnoisew_{t+1}}^{p}\bigr] + n^{-1}\, \PE^{1/p}\parentheseDeuxLigne{\norm{\sum_{t=n/2}^{n-1} \zmfuncAw[t+1](\theta_t - \thetas)}^p} \eqsp.
\]
The first term of the above sum can be controlled by directly applying Pinelis' version of Rosenthal's inequality \cite[Theorem 4.3]{pinelis_1994}:
\[
\textstyle
\PE^{1/p}\bigl[\norm{\sum_{t=n/2}^{n-1}\funnoisew_{t+1}}^{p}\bigr] \lesssim p^{1/2} n^{1/2} \sqrt{\trace{\noisecov}} + p \supconsteps\eqsp.
\]
It remains to bound $\PE^{1/p}\parentheseDeuxLigne{\norm{\sum_{t=n/2}^{n-1} \zmfuncAw[t+1](\theta_t - \thetas)}^p}$. Note that the sequence $\{\zmfuncAw[t+1](\theta_t - \thetas)\}$ is a martingale-difference w.r.t. $\mathcal{F}_t = \sigma(Z_k, k \leq t)$. A further application of Rosenthal's inequality thus shows that

\vspace{-12pt}
\begin{small}
\begin{align}
&\PE^{1/p}\biggl[\bigg\Vert\sum_{t=n/2}^{n-1} \zmfuncAw[t+1](\theta_t - \thetas) \bigg\Vert^p\biggr] \\
&\qquad \lesssim p^{1/2} \PE^{1/p}\biggl[\biggl(\sum_{t=n/2}^{n-1} \CPE{\normop{\zmfuncAw[t+1](\theta_t - \thetas)}^{2}}{\mathcal{F}_t}\biggr)^{p/2}\biggr] + p\, \PE^{1/p}\left[\max_{t}\normop{\zmfuncAw[t+1](\theta_t - \thetas)}^{p}\right] \\
&\qquad \leq p^{1/2} \bConst{A} \left(\sum_{t=n/2}^{n-1} \PE^{2/p}\left[\normop{\theta_t - \thetas}^{p}\right]\right)^{1/2} + p \bConst{A} \, \PE^{1/p}\left[\max_{t}\normop{\theta_t - \thetas}^{p}\right]\eqsp.
\end{align}
\end{small}
\!\!Now, applying the last iterate bound \Cref{th:LSA_last_iterate_refined}-\eqref{eq:Rosenthal_LSA}, and using that $\alpha a \leq 1/2$, we get
\begin{align}
&p^{1/2} \bConst{A} \biggl(\sum_{t=n/2}^{n-1} \PE^{2/p}[\normop{\theta_t - \thetas}^{p}]\biggr)^{1/2}
\lesssim \frac{\conststab[p+\log{n}] p^{1/2} \bConst{A} (1- \alpha a)^{n/2} \norm{\theta_0 - \thetas}}{\sqrt{\alpha a}} \\
& \qquad + \frac{\conststab[p+\log{n}] \bConst{A} p \sqrt{\alpha n \trace{\noisecov}}}{\sqrt{a}} + \conststab[p+\log{n}] \bConst{A} \alpha p^{3/2} n^{1/2} \supconsteps\eqsp.
\end{align}
Moreover, a further application of \Cref{th:LSA_last_iterate_refined}-\eqref{eq:Rosenthal_LSA} together with $n^{1/\log{n}} \leq \rme$ yield
\begin{align}
p \bConst{A} \, \PE^{1/p}[\max_{t}\normop{\theta_t - \thetas}^{p}]
&\leq p \bConst{A} \left(\sum_{t=n/2}^{n}\PE[\normop{\theta_t - \thetas}^{p+\log{n}}]\right)^{1/(p+\log{n})} \\
&\lesssim p \bConst{A} n^{1/(p+\log n)} \max_{n/2 \leq t < n}  \PE^{1/(p+\log{n})}[\normop{\theta_t - \thetas}^{p+\log{n}}] \\
&\lesssim \conststab[p+\log{n}] p \bConst{A} (1- \alpha a)^{n/2} \norm{\theta_0 - \thetas} + \conststab[p+\log{n}] p \bConst{A} \frac{\sqrt{\alpha (p + \log n) \trace{\noisecov}}}{a} \\
&\quad + \conststab[p+\log{n}] p \bConst{A} \alpha (p + \log n) \supconsteps \eqsp \eqsp.
\end{align}
Now it remains to combine the obtained bounds, and the statement follows. The result of \Cref{th:LSA_PR_error} follows from a simple observation that $\alpha (p + \log{n}) \leq 1/2$ under \Cref{assum:exp_stability}($p+\log{n}$) for $\alpha \in (0;\alpha_{p + \log n,\infty}]$.
\end{proof}

\section{Proofs of TD learning of \Cref{sec:td_learning} }
\label{sec:td_proofs}
\subsection{Proof of \Cref{prop:condition_check}}
\begin{proof}
Under \Cref{assum:feature_design}, it is easily seen that $\norm{\funcAw_1} \leq (1 + \gamma)$ almost surely, which implies $\norm{\bA} \leq (1 + \gamma)$. The remaining bounds follow from
\begin{align}
\supconsteps &= \sup_{z \in \msz}\normop{\funcnoise{z}} = \sup_{z = (s,s')}\normop{(\funcA{z} - \bA)\thetas - (\funcb{z} - \barb)} \leq 2 (1+\gamma)(\norm{\thetas} + 1)\eqsp,\\
\trace{\noisecov}
&= \PE[\norm{(\funcAw_1 - \bA) \thetas - (\funcbw_1 - \barb)}^2] \leq 2 \thetas^{\top} \PE[\funcAw_{0}^{\top} \funcAw_{0}] \thetas + 2\PE[r^2(s_0)\trace{\varphi(s_0) \varphi^{\top}(s_0)}] \\
&\leq 2  (1+\gamma)^2 \thetas^{\top} \covfeat \thetas + 2 \leq 2  (1+\gamma)^2 \left(\norm{\thetas}[\covfeat]^2 + 1\right)\eqsp,
\end{align}
and the statement follows.
\end{proof}

\subsection{Proof of \Cref{prop:stability_check_td}}
\label{subsec:stability_check_td}
In this subsection we obtain a new, refined bounds on the transient term $\utheta_{n} = \ProdBa_{1:n} \{ \theta_0 - \thetas \}$ appearing in the error decomposition \eqref{eq:LSA_recursion_expanded} in case of TD(0) algorithm. Recall that in case of the general LSA algorithm we have to refer to the matrix product stability result of \Cref{prop:general_expectation}, which is based on the framework suggested by \cite{huang2020matrix}. The interplay between step size $\alpha_{\infty,p}$ and maximal controlled moment $p$ (which roughly can be written as $\alpha_{\infty,p} \lesssim 1/p$) is in general unavoidable. The respective $1-$dimensional counterexample is provided in \cite[Example~1]{durmus2021tight}. At the same time, the general $L_p$-stability of the random matrix product appears to induce some undesirable phenomenons. First, it induces the additional $d^{1/p}$ factor in the r.h.s. of the bound \eqref{eq:concentration_iid_matrix}. Such a dependence requires to introduce additional (logarithmic) dependence of the dimension $d$ in the step size $\alpha$ in order to remove the $d^{1/p}$ factor in the r.h.s..
\par
Second, and more important, the trade-off between $\norm{\Id - \alpha \bA} \leq 1-\alpha a$ and upper bounds for fluctuation term $\alpha \PE^{1/p}[\norm{\funcAw - \bA}^{p}]$ requires that the step size $\alpha$ scales with some instance-dependent quantities, related with the matrix $\bA$. Typically this means that the resulting rate-optimal algorithm is not really implementable, as $\bA$ is not accessible in practice.
\par
This drawback is shared by most of the recent papers on the subject, see e.g. \cite[Theorem~1]{li2023sharp}, where the maximal allowed step size $\alpha$ scales with $\lambda_{\min}(\Sigma)$. Our subsequent analysis allows us to eliminate this drawback. Recall that for any $1 \leq j \leq n$ we set $\mathcal{F}_{j} = \sigma(Z_i, 1 \leq i \leq j)$ and $\mathcal{F}_{0} = \emptyset$. Then the exponential stability property of \Cref{prop:stability_check_td} will follow from the following general result:
\begin{theorem}
\label{lem:refined_product_stability}
Let $\sequence{\theta}[k][\nset]$ be a sequence of TD(0) updates generated by \eqref{eq:LSA_procedure_TD} under \Cref{assum:generative_model} and \Cref{assum:feature_design}. Then, for any $n \in \nset$, $1 \leq j \leq n$, $p \geq 2$, step size $\alpha \in (0;\frac{1-\gamma}{128 p}]$, and any $\xi_{j-1}$ being a $d$-dimensional $\mathcal{F}_{j-1}$-measurable random vector, it holds $\PP$-a.s. that
\begin{equation}
\label{eq:better_bound_matr_product}
\CPE{\normop{\ProdBa_{j:n} \xi_{j-1}}^{p}}{\mathcal{F}_{j-1}} \leq (1-\alpha p (1-\gamma) \lambda_{\min}/2)^{n-j}\norm{\xi_{j-1}}^{p}\eqsp.
\end{equation}
In particular, for any $\theta_0 \in \rset^{d}$ we obtain that
\begin{equation}
\label{eq:better_bound_td_transient}
\PE^{1/p}[\normop{\ProdBa_{1:n}(\theta_0 - \thetas)}^{p}] \leq (1-\alpha (1-\gamma) \lambda_{\min}/2)^{n-j} \norm{\theta_0 - \thetas}\eqsp.
\end{equation}
\end{theorem}
\begin{proof}
Note that it is enough to prove the bound \eqref{eq:better_bound_matr_product} for $p = 2^{s}$, $s \in \nset$, since otherwise we can find the nearest dyadic power $q \geq p$ and use the Lyapunov inequality. Note that we increase the power of $p$ by no more than a factor of $2$ in such a case.
\par
Now we consider the case $p = 2^{s}$, $s \in \nset$. Then, expanding the $p$-power of the norm, we get
\[
\normop{\ProdBa_{j:n} \xi_{j-1}}^{p} = \bigl(\xi_{j-1}^{\top} \{\ProdBa_{j:n}\}^{\top} \ProdBa_{j:n} \xi_{j-1}\bigr)^{p/2} =
(\eta_{n-1}^{\top}(\Id - \alpha \funcAw_{n})^{\top} (\Id - \alpha \funcAw_{n}) \eta_{n-1})^{p/2}\eqsp,
\]
where we have introduced a vector $\eta_{n-1} = \ProdBa_{j:n-1}\xi_{j-1}$. Note that a vector $\eta_{n-1}$ is $\mathcal{F}_{n-1}$-measurable, and thus, combining \Cref{lem:matrix_products} and \Cref{lem:matrix_product_deterministic}, we get
\begin{align}
\CPE{\normop{\ProdBa_{j:n} \xi_{j-1}}^{p}}{\mathcal{F}_{j-1}} &= \CPE{\CPE{(\eta_{n-1}^{\top}(\Id - \alpha \funcAw_{n})^{\top} (\Id - \alpha \funcAw_{n}) \eta_{n-1})^{p/2}}{\mathcal{F}_{n-1}}}{\mathcal{F}_{j-1}} \\
&\leq (1-\alpha p (1-\gamma) \lambda_{\min}/2) \CPE{\normop{\ProdBa_{j:n-1} \xi_{j-1}}^{p}}{\mathcal{F}_{j-1}}\eqsp,
\end{align}
and the bound \eqref{eq:better_bound_matr_product} follows by backward induction in $n$. In order to get the bound \eqref{eq:better_bound_td_transient}, it remains to combine \eqref{eq:better_bound_matr_product} together with the fact that $g(x,p) = (1-px)^{1/p}$ monotonically decreases in $p$ for $p \geq 1$ and $ 0 < x < 1$.
\end{proof}

The stability result of \Cref{lem:refined_product_stability} favorably compares to the one of \Cref{prop:general_expectation}. First, we removed an artificial $d^{1/p}$ factor in the r.h.s. of the bound. Second, new stability threshold for $\alpha$ is \emph{computable} and does not contain any instance-independent quantities.
\par

Below we provide some useful auxiliary technical lemmas required for the proof of \Cref{lem:refined_product_stability}.
\begin{lemma}
\label{lem:matrix_products}
Let $B = B^{\top} \geq 0$, $B \in \rset^{d \times d}$ be a symmetric positive definite matrix and $u \in \rset^{d}$ be some vector. Then, for any $s \in \nset$ and $p = 2^{s}$, it holds that
\begin{equation}
\label{eq:norm_bound_deterministic}
\bigl(u^{\top} B u\bigr)^{p} \leq \norm{u}^{2p-2} u^{\top} B^{p} u\eqsp.
\end{equation}
\end{lemma}
\begin{proof}
We will proof the statement by induction in $s \in \nset$. The statement obviously holds for $s = 0$. For $s = 1$ (resp., $p = 2$), we aim to prove that
\begin{equation}
\label{eq:quadr_form_bound_p_2}
\bigl(u^{\top} B u\bigr)^{2} = u^{\top} B u u^{\top} B u \leq \norm{u}^2 u^{\top} B^{2} u \eqsp,
\end{equation}
and the statement follows from the bound $B u u^{\top} B \leq \norm{u}^2 B^{2}$. Let us provide the detailed proof of last inequality. We aim to check that for any $y \in \rset^{d}$ it holds that
\[
y^{\top} B u u^{\top} B y \leq \norm{u}^2 y^{\top} B^{2} y\eqsp.
\]
Note that, since $B$ is symmetric and positive definite, $B = U \Lambda U^{\top}$ with diagonal matrix $\Lambda = \diag\{\lambda_1,\ldots,\lambda_d\}$ and orthogonal matrix $U$. Hence, the previous inequality is equivalent to
\[
y^{\top} U \Lambda U^{\top} u u^{\top} U \Lambda U^{\top} y \leq \norm{u}^2 y^{\top} U \Lambda^2 U^{\top} y\eqsp.
\]
Setting $z = U^{\top} y$ and $v = U^{\top} u$, we have from the previous bound
\[
z^{\top} \Lambda v v^{\top} \Lambda z \leq \norm{v}^2 z^{\top} \Lambda^2 z\eqsp.
\]
Writing the previous bound in a coordinate form, we obtain that
\[
\left(\sum_{i=1}^{d}\lambda_i z_i v_i\right)^2 \leq \left(\sum_{i=1}^{d}v_i^2\right)\left(\sum_{i=1}^{d}\lambda_i^2z_i^2\right)\eqsp,
\]
which holds due to Cauchy-Schwartz inequality, and \eqref{eq:quadr_form_bound_p_2} holds.
\par
Suppose now that the inequality \eqref{eq:norm_bound_deterministic} holds for some $p = 2^{s}$. Then
\[
\bigl(u^{\top} B u\bigr)^{2p} = \bigl(u^{\top} B u\bigr)^{p} \bigl(u^{\top} B u\bigr)^{p} \leq \norm{u}^{4p-4} u^{\top} B^{p} u u^{\top} B^{p} u \leq \norm{u}^{4p-2} u^{\top} B^{2p} u\eqsp,
\]
and the statement follows.
\end{proof}

Now we provide a key statement on the in-expectation contraction of 1-step-ahead random matrix $\funcAw$ corresponding to the TD(0) algorithm.

\begin{lemma}
\label{lem:matrix_product_deterministic}
Let $\funcAw = \varphi(s)\{\varphi(s) - \gamma \varphi(s')\}^{\top}$ be a random TD update matrix defined in \eqref{eq:matr_A_def}, where $s' \sim \PMDP^{\pi}(\cdot|s)$, and $s \sim \mu$. Then, for any $p \in \nset$ and $\alpha \in (0;\frac{1-\gamma}{64 p}]$, it holds that
\begin{equation}
\label{eq:matr_prod_determ}
\PE\bigl[\{(\Id - \alpha \funcAw)^{\top}(\Id - \alpha \funcAw)\}^{p}\bigr] \preceq \Id - (1/2) \alpha p (1-\gamma) \covfeat \eqsp.
\end{equation}
\end{lemma}
\begin{proof}
Consider the (random) matrix $(\Id - \alpha \funcAw)^{\top}(\Id - \alpha \funcAw)$. Note that it is symmetric, and, introducing matrix $\funcBw = \funcAw + \funcAw^{\top} - \alpha \funcAw^{\top} \funcAw$, we get that
\[
(\Id - \alpha \funcAw)^{\top}(\Id - \alpha \funcAw) = \Id - \alpha \funcBw\eqsp.
\]
Using  \Cref{lem:B_matr_bounds}, it holds that for any $k \in \nset$,
\[
\PE[\funcBw] \succeq (1-\gamma)\covfeat\eqsp, \quad \PE[\funcBw^{k}] \preceq \frac{13}{12} \cdot 4^{k} \covfeat \eqsp.
\]
Thus, expanding the brackets, we get
\[
\PE[(\Id - \alpha \funcBw)^{p}] \preceq \Id - \alpha p \PE[\funcBw] + \sum_{k=2}^{p}\alpha^{k}\binom{p}{k}\PE[\funcBw^{k}] \preceq \Id - \alpha p (1-\gamma) \covfeat + \frac{13}{12} \cdot \biggl(\sum_{k=2}^{p}(4\alpha)^{k}\binom{p}{k}\biggr)\covfeat\eqsp.
\]
Since we know that $\alpha p \leq (1-\gamma)/64$, we can bound
\[
\sum_{k=2}^{p}(4\alpha)^{k}\binom{p}{k} \leq \sum_{k=2}^{p}(4\alpha p)^{k} \leq \frac{16\alpha^2p^2}{1 - 4\alpha p} \leq \frac{16}{15} \cdot 16 \alpha^2p^2 \leq \frac{16}{15} \cdot \alpha p (1-\gamma) /4\eqsp.
\]
Thus the combination of above bounds together with $\frac{16}{15} \cdot \frac{13}{12} < 2$ imply that
\[
\PE[(\Id - \alpha \funcBw)^{p}] \preceq \Id - (1/2) \alpha p (1-\gamma) \covfeat\eqsp,
\]
and the statement follows.
\end{proof}

Now we provide a technical lemma on the behaviour of the symmetrized random matrix update $(\Id - \alpha \funcAw)^{\top}(\Id - \alpha \funcAw)$, where $\funcAw$ is defined in \eqref{eq:matr_A_def}. This lemma generalize the results presented in \cite[Lemma~5]{patil2023finite}.

\begin{lemma}
\label{lem:B_matr_bounds}
For the random matrix $\funcAw$ defined in \eqref{eq:matr_A_def} and $\funcBw = \funcAw + \funcAw^{\top} - \alpha \funcAw^{\top} \funcAw$, $p \in \nset$ and step size $\alpha \in (0;\frac{1-\gamma}{(1+\gamma)^2}]$ it holds that
\begin{equation}
\label{eq:matr_B_bounds}
\begin{split}
\PE[\funcBw] &\succeq (1-\gamma)\covfeat\eqsp, \\
\PE[\funcBw^{p}] &\preceq \frac{13}{12} \cdot 4^{p} \covfeat\eqsp.
\end{split}
\end{equation}
\end{lemma}
\begin{proof}
With the definition of $\funcAw$, we get that
\begin{align}
\funcAw + \funcAw^{\top}
&= \varphi(s)\{\varphi(s) - \gamma \varphi(s')\}^{\top} + \{\varphi(s) - \gamma \varphi(s')\}\varphi(s)^{\top} \\
&= 2 \varphi(s)\varphi(s)^{\top} - \gamma \{\varphi(s)\varphi(s')^{\top} + \varphi(s')\varphi(s)^{\top}\} \label{eq:a_at_bound} \\
&\succeq (2-\gamma) \varphi(s)\varphi(s)^{\top} - \gamma \varphi(s')\varphi(s')^{\top} \eqsp,
\end{align}
where we used an elementary inequality $u v^{\top} + v u^{\top} \preceq (uu^{\top}+vv^{\top})$ valid for any $u,v \in \rset^{d}$. Similarly, with elementary algebra, we obtain
\begin{equation}
\label{eq:ata_bound}
\begin{split}
\funcAw^{\top} \funcAw
&= \{\varphi(s) - \gamma \varphi(s')\}\varphi(s)^{\top}\varphi(s)\{\varphi(s) - \gamma \varphi(s')\}^{\top} \\
&= \norm{\varphi(s)}^{2}\{\varphi(s) - \gamma \varphi(s')\}\{\varphi(s) - \gamma \varphi(s')\}^{\top} \\
&= \norm{\varphi(s)}^{2}\{\varphi(s)\varphi(s)^{\top} + \gamma^2 \varphi(s')\varphi(s')^{\top} - \gamma(\varphi(s)\varphi(s')^{\top} + \varphi(s')\varphi(s)^{\top})\} \\
&\overset{(a)}{\preceq} (1+\gamma)\varphi(s)\varphi(s)^{\top} + \gamma(1+\gamma)\varphi(s')\varphi(s')^{\top}\eqsp,
\end{split}
\end{equation}
where in (a) we additionally used that $\norm{\varphi(s)} \leq 1$ and
\[
-(uu^{\top}+vv^{\top}) \preceq u v^{\top} + v u^{\top} \preceq (uu^{\top}+vv^{\top})
\]
for any $u,v \in \rset^{d}$. Combining the bounds above yields that for $0 \leq \alpha \leq \frac{1-\gamma}{(1+\gamma)^2}$ it holds that
\[
\PE[\funcBw] \succeq 2(1-\gamma)\covfeat - \alpha (1+\gamma)^2 \covfeat \succeq (1-\gamma)\covfeat \eqsp,
\]
and the first part of \eqref{eq:matr_B_bounds} is proved. To prove the second part it remains to notice that, for $p \in \nset$, and $0 \leq \alpha \leq \frac{1-\gamma}{(1+\gamma)^2}$, it holds that
\[
\funcBw^{p} = \funcBw^{p-2} \funcBw^{2} \preceq \normop{\funcBw}^{p-2} \funcBw^{2}\eqsp,
\]
and
\[
\normop{\funcBw} = \normop{\funcAw + \funcAw^{\top} - \alpha \funcAw^{\top} \funcAw} \leq 2(1+\gamma) + \alpha (1+\gamma)^2 \leq 3 + \gamma \leq 4\eqsp.
\]
Now it remains to analyze the expectation of the matrix $\funcBw^{2}$, which is symmetric and positive semi-definite:
\begin{align}
\funcBw^{2} 
&= (\funcAw + \funcAw^\top - \alpha \funcAw^\top \funcAw) (\funcAw + \funcAw^\top - \alpha \funcAw^\top \funcAw) \\
&= (\funcAw + \funcAw^\top)^2 - \alpha\left[ (\funcAw + \funcAw^\top) \funcAw^\top \funcAw + \funcAw^\top \funcAw (\funcAw + \funcAw^\top) \right] + \alpha^2 (\funcAw^\top \funcAw)^2.
\end{align}
We start from the first term. With the simple algebra and the definition of $\funcAw$, we obtain that 
\begin{align}
    \funcAw^2 &= \varphi(s) (\varphi(s) - \gamma \varphi(s'))^\top \varphi(s) (\varphi(s) - \gamma \varphi(s'))^\top = \langle \varphi(s), \varphi(s) - \gamma \varphi(s') \rangle \funcAw \eqsp,  \label{eq:lm5_aa} \\
    (\funcAw^\top)^2 &= (\varphi(s) - \gamma \varphi(s')) \varphi(s)^{\top} (\varphi(s) - \gamma \varphi(s')) \varphi(s)^{\top} = \langle \varphi(s), \varphi(s) - \gamma \varphi(s') \rangle \funcAw^{\top} \eqsp,  \\
    \funcAw \funcAw^\top &= \varphi(s) (\varphi(s) - \gamma \varphi(s'))^\top  (\varphi(s) - \gamma \varphi(s')) \varphi(s)^\top = \norm{\varphi(s) - \gamma \varphi(s')}^2 \phi(s) \phi(s)^\top\eqsp, \\
    \funcAw^\top \funcAw  &= (\varphi(s) - \gamma \varphi(s')) \varphi(s)^\top \varphi(s) (\varphi(s) - \gamma \varphi(s'))^\top \\
    &= \norm{\varphi(s)}^2 (\varphi(s) - \gamma \varphi(s')) (\varphi(s) - \gamma \varphi(s'))^{\top}\eqsp.
\end{align}
Additionally, in expectation we have the following relations, that follows from \eqref{eq:a_at_bound} and \eqref{eq:ata_bound}
\begin{equation}\label{eq:expectation_at_a_ata_bound}
    \PE\left[\funcAw + \funcAw^\top\right] \preceq 2(1+\gamma) \covfeat, \quad \PE\left[ \funcAw^\top  \funcAw \right] \leq (1+\gamma)^2 \covfeat.
\end{equation}
Using this relations, and using the fact that 
\begin{align}
(\funcAw + \funcAw^\top)^2 = \funcAw^2 +  \funcAw \funcAw^\top + \funcAw^\top \funcAw + (\funcAw^\top)^2\eqsp,
\end{align}
we obtain that 
\[
\PE\left[ (\funcAw + \funcAw^\top)^2 \right] \preceq 4 (1+\gamma)^2  \covfeat.
\]
For the term $(\funcAw^\top \funcAw)^2$ we obtain that 
\begin{align}
(\funcAw^\top \funcAw)^2 = \funcAw^\top \left(\funcAw\funcAw^\top\right) \funcAw 
&= \norm{\varphi(s) - \gamma \varphi(s')}^2 \funcAw^\top \phi(s) \phi(s)^\top \funcAw \\
&\overset{(a)}{=} \norm{\varphi(s) - \gamma \varphi(s')}^2 \norm{\varphi(s)}^2 \funcAw^\top \funcAw\eqsp.
\end{align}
Here the last identity (a) follows from the particular form of the matrix $\funcAw = \varphi(s)\{\varphi(s) - \gamma \varphi(s')\}^{\top}$. Therefore, using the representation \eqref{eq:ata_bound}, we obtain that 
\[
\PE\left[(\funcAw^\top \funcAw)^2 \right] \preceq (1+\gamma)^4 \covfeat \eqsp.
\]
The above bounds together with the Cauchy-Schwartz inequality imply that, for any $\epsilon > 0$, we have 
\begin{align}
\funcBw^{2} \preceq (1+\epsilon) (\funcAw + \funcAw^\top)^2 + (1 + 1/\epsilon)\alpha^{2} (\funcAw^\top \funcAw)^2\eqsp,
\end{align}
which implies that 
\begin{align}
\PE[\funcBw^2] \preceq 4 (1+\epsilon) (1+\gamma)^2 \covfeat + (1 + 1/\epsilon)\alpha^{2} (1+\gamma)^4 \covfeat\eqsp.
\end{align}
Thus, setting $\epsilon = 1/12$, and using that $\gamma \in [0;1]$, we get that 
\begin{align}
\PE[\funcBw^2] \preceq (52/3) \covfeat = (13/12) 16 \covfeat\eqsp.
\end{align}
As a result
\[
\PE[\funcBw^2] \preceq (13/12) 16 \covfeat \Rightarrow \PE[\funcBw^p] \preceq 4^{p-2} \PE[\funcBw^2] \preceq (13/12) 4^p \covfeat.
\]

\end{proof}

\subsection{Missing results from \Cref{sec:td_learning}}
\label{sec:td_missing}
We begin this section from instantiating \Cref{th:td_lsa_pr_2nd_moment} for the sequence $\{\theta_k\}$ which corresponds to  TD(0) algorithm. We use that $\conststab[2] = 1$, $\trace{\noisecov} \leq 1 + \norm{\thetas}[\covfeat]^2$, $a = \lambda_{\min}(1-\gamma)/2$. Then we get

\begin{theorem}
\label{th:td_lsa_pr_2nd_moment_appendix}
Assume \Cref{assum:generative_model} and \Cref{assum:feature_design}.
Let $\sequence{\theta}[k][\nset]$ be a sequence of TD(0) updates generated by \eqref{eq:LSA_procedure_TD}.  Then for any $n \geq 2$, $\alpha \in \ocint{0;\frac{1-\gamma}{256}}$, and $\theta_0 \in \rset^d$, it holds that
\begin{align}
\PE^{1/2}[ \norm{\prtheta_{n} - \thetas}[\covfeat]^2] &\lesssim  \frac{\norm{\thetas}[\covfeat] + 1}{\sqrt{\lambda_{\min} n} (1-\gamma) }\left(1 + \frac{\sqrt{\alpha}}{\sqrt{(1-\gamma) \lambda_{\min}}}\right) + \frac{\norm{\thetas}[\covfeat] + 1}{\sqrt{\alpha} (1-\gamma)^{3/2} \lambda_{\min} n} \\
&+ \left(\frac{1}{\alpha n(1-\gamma)\lambda_{\min}^{1/2}} + \frac{1}{\sqrt{\alpha} n (1-\gamma)^{3/2} \lambda_{\min}}\right) \left(1 - \frac{\alpha (1-\gamma) \lambda_{\min}}{2} \right)^{n/2} \norm{\theta_0 - \thetas}\eqsp.
\end{align}
\end{theorem}

Similarly to the discussion above, we can state the respective $p$-moment bound for the case of TD(0) algorithm. This theorem is an adaptation of \Cref{th:LSA_PR_error} (see also \Cref{th:lsa_pr_error_p_moment}).

\begin{theorem}
\label{th:td_pr_pth_moment}
Assume \Cref{assum:generative_model} and \Cref{assum:feature_design}.
Let $\sequence{\theta}[k][\nset]$ be a sequence of TD(0) updates generated by \eqref{eq:LSA_procedure_TD}.
Then for any $p \geq 2$, $n \geq 2$, and step size
\[
\alpha \in \biggl(0; \frac{1-\gamma}{128(p + \log{n})}\biggr]\eqsp,
\]
we have that
\begin{equation}
\label{eq:pr_theta_p_moment_bound}
\begin{split}
\PE^{1/p}[\norm{(\prtheta_{n} - \thetas)}[\covfeat]^p]
&\lesssim \frac{p^{1/2} (\norm{\thetas}[\covfeat] + 1)}{n^{1/2} (1-\gamma) \lambda_{\min}^{1/2}}\left(1 + \frac{\sqrt{\alpha p} + \alpha p}{\sqrt{(1-\gamma) \lambda_{\min}}} \right) + \frac{p (\norm{\thetas}[\covfeat] + 1)}{n (1-\gamma)^{3/2} \lambda_{\min}}\left(1 + \frac{1}{\sqrt{\alpha p}}\right) \\
&+ \left(1 - \frac{\alpha (1-\gamma) \lambda_{\min}}{2}\right)^{n/2} \left( (p + \log(n))^{1/2} + \frac{p}{\sqrt{\lambda_{\min}}} \right)
\frac{\{p + \log(n)\}^{1/2}}
{(1-\gamma)^2 \sqrt{\lambda_{\min}} n} \norm{\theta_0 - \thetas} \eqsp.
\end{split}
\end{equation}
\end{theorem}
The respective high-probability bound can be written as follows:

\begin{corollary}
\label{cor:td_pr_pth_moment}
Assume \Cref{assum:generative_model} and \Cref{assum:feature_design}. Let $\sequence{\theta}[k]$ be a sequence of TD(0) updates generated by \eqref{eq:LSA_procedure_TD}.  Fix $\delta \in (0;1/\rme)$. Then, for the step- and sample size
{\small
\[
\alpha = \frac{1-\gamma}{128\log{(n/\delta)}}, \quad n \geq \frac{\log{(1/\delta)}}{(1-\gamma)^2}
\]
}
it holds with probability at least $1-\delta$ that 
{\small 
\begin{equation}
\label{eq:deviation_bound_pth_moment}
\norm{\prtheta_{n} - \thetas}[\covfeat] 
\lesssim \frac{(\norm{\thetas}[\covfeat] + 1)\sqrt{\log{(1/\delta)}}}{n^{1/2} (1-\gamma) \lambda_{\min}} + \bigl(1 - \frac{(1-\gamma)^2 \lambda_{\min}}{128 \log{(n/\delta)}}\bigr)^{n/2} \frac{\norm{\theta_0 - \thetas} \log^{3/2}{(n/\delta)}}{(1-\gamma)^2 \lambda_{\min} n}\eqsp.
\end{equation}
}
\end{corollary}

\subsection{Proof of stability bound \eqref{eq:old_stability_threshold} based on matrix stability argument}
\label{sec:old_stability_bound}
In the previous section we have presented a stability result \Cref{lem:refined_product_stability}, which allows for maximal step size in the constant-step size $TD(0)$ algorithm $\alpha_{\infty,p}$ of the form
\[
\alpha_{\infty,p} = \frac{1-\gamma}{128 p}\,.
\]
In this subsection we show that such type of result can not be readily obtained from existing results on the stability of random matrix product \cite{huang2020matrix}.
\par
We first introduce some matrix notations. For the matrix $\MatB \in \rset^{d \times d}$ we denote by $( \sigma_\ell(\MatB) )_{ \ell=1 }^d$ its singular values. For $p \geq 1$, the Shatten $p$-norm is denoted by $\norm{\MatB}[p] = \{\sum_{\ell=1}^d \sigma_\ell^p (\MatB)\}^{1/p}$. For $p, q \geq 1$ and a random matrix $\X$ we write $\norm{\X}[p,q] = \{ \PE[\norm{\X}[p]^q] \}^{1/q}$. Then it is easily seen that
\[
\PE^{1/q}[\norm{\X}^{q}] \leq \norm{\X}[\qexponent,\ppexponent]\eqsp,
\]
and one can control an operator norm of the matrix with its Shatten norm of an appropriate order. Now we state the following result from \cite[Proposition~2]{durmus2021tight}.
\begin{proposition}
\label{prop:general_expectation}
Let $\sequence{\Y}[\ell][\nset]$ be a sequence on independent matrices, $\Y_{\ell} \in \rset^{d \times d}$ and $\Q$ be a positive definite matrix. Assume that for each $\ell \in \nset$ there exist $m_\ell \in (0,1)$  and $\sigma_{\ell} > 0$ such that \(\norm{\PE[\Y_\ell]}[\Q]^2  \leq 1 - m_\ell\) and \(\norm{\Y_\ell - \PE[\Y_\ell]}[\Q] \leq \sigma_{\ell}\) almost surely.  Define $\Zbf_n = \prod_{\ell = 0}^n \Y_\ell= \Y_n \Zbf_{n-1}$, for $n \geq 1$ with some (deterministic) matrix $\Zbf_0 \in \rset^{d \times d}$. Then, for any $2 \le q \le p$ and $n \geq 1$,
\begin{equation}
\label{eq:gen_expectation}
\norm{\Zbf_n}[p,q]^2 \leq \qcond \prod_{\ell=1}^n (1- m_\ell + (p-1)\sigma_{\ell}^2) \norm{\Q^{1/2}\Zbf_0 \Q^{-1/2}}[p, q]^2 \eqsp,
\end{equation}
where  $\qcond = \lambda_{\mathsf min}^{-1}(\Q)\lambda_{\mathsf max}(\Q)$.
\end{proposition}

Note that the result of \Cref{prop:general_expectation} is generic in a sense that it allows us an additional degree of freedom in the choice of the contracting matrix norm $\norm{\cdot}[\Q]$. An almost sure bound on $\norm{\Y_\ell - \PE[\Y_\ell]}[\Q]$ can be generalized to a moment-type bound, with the same shape of the bound in \eqref{eq:gen_expectation}. The main drawback of this technique is an inevitable trade-off between $m_{\ell}$ and $(p-1)\sigma_{\ell}^2$ factors, which directly influences the speed with which $\norm{\Zbf_n}[p,q]^2$ decays to $0$.
\par
Now we aim to apply \Cref{prop:general_expectation} to check the assumption \Cref{assum:exp_stability} for the TD(0) algorithm.

\begin{lemma}
\label{lem:stability_TD}
Let $\sequence{\theta}[k][\nset]$ be a sequence of TD(0) updates generated by \eqref{eq:LSA_procedure_TD} under \Cref{assum:generative_model} and \Cref{assum:feature_design}. Then this update scheme satisfies assumption \Cref{assum:exp_stability}$(p)$ with
\begin{equation}
\label{eq:stability_threshold_td_old}
a = \frac{(1-\gamma) \lambda_{\min}(\covfeat)}{2}\,, \eqsp \alpha_{p,\infty} = \frac{1-\gamma}{128p} \wedge \frac{(1-\gamma) \lambda_{\min}(\covfeat)}{64 p}\eqsp, \quad \conststab[p] = d^{1/p}\,.
\end{equation}
\end{lemma}
\begin{proof}
We aim to apply here the result of \Cref{prop:general_expectation} with $\Y_\ell = \Id - \alpha \funcAw_{\ell}$ and $\Zbf_n = \ProdBa_{1:n}$. Towards this aim, note that \Cref{lem:matrix_product_deterministic} implies that, with $\funcAw = \varphi(s)\{\varphi(s) - \gamma \varphi(s')\}^{\top}$ being a random TD update matrix defined in \eqref{eq:matr_A_def}, we have
\begin{align}
\normop{\Id - \alpha \bA} = \normop{\PE[\Id - \alpha \funcAw]} 
&\leq \sqrt{\normop{\PE[(\Id - \alpha \funcAw)^{\top}(\Id - \alpha \funcAw)\}]}} \\
&\leq \sqrt{1 - \alpha (1-\gamma) \lambda_{\min}(\covfeat)} \\
&\leq 1 - (1/2)\alpha (1-\gamma) \lambda_{\min}(\covfeat)\eqsp,
\end{align}
which holds for $\alpha \in (0;\frac{1-\gamma}{128})$. Moreover,
\[
\normop{\alpha(\funcAw - \bA)} \leq \alpha \normop{\varphi(s) (\varphi(s) - \gamma \varphi(s'))^\top } + \alpha \normop{\PE\left[ \varphi(s) (\varphi(s) - \gamma \varphi(s'))^\top \right]}  \leq 2(1+\gamma)\alpha \leq 4\alpha \eqsp.
\]
Hence, setting $a = (1-\gamma) \lambda_{\min}(\covfeat)$, the assumptions of \Cref{lem:matrix_product_deterministic} are satisfied with
\[
\sigma_{\ell} = 4\alpha\eqsp, m_{\ell} = \alpha a /2\eqsp.
\]
Hence, applying the result of \Cref{lem:matrix_product_deterministic} with $Q = \Id$, $\Zbf_0 = \Id$, we get
\[
\PE^{1/q}\left[ \normop{\ProdBa_{1:n}}^{q} \right]
\leq  \norm{\ProdBa_{1:n}}[p,q]
\leq d^{1/p} (1 - \alpha a + 16 (p-1)  \alpha^2)^{n/2}\eqsp.
\]
Now we have to balance the terms $\alpha a /2$ and $16 (p-1)  \alpha^2$, which yields the scaling of $\alpha$ with $a$ (and, hence, with $\lambda_{\min}(\covfeat)$). In particular, setting $\alpha = \frac{a}{32 p}$, we get the statement of the Lemma.
\end{proof}

\section{Proofs of \Cref{sec:lower_bounds}}
\label{sec:app:lower_bounds}
In this section we need to introduce an additional assumptions which relates matrices $\G$, $\bA$, and (random) matrices $\funcAw_i$ for $i \in \{1,\ldots,n\}$.
\begin{assumptionC}
\label{assum:g_matr}
There exist such symmetric positive-definite matrix $\G = \G^{\top} > 0$ and constants $\matrbound > 0$, $\randmbound > 0$, $\tracebound > 0$, such that
\begin{enumerate}[(i)]
    \item for the system matrix $\bA$ it holds that
    \[
    \G^{1/2} \bA^{-\top} \G \bA^{-1} \G^{1/2} \preceq \matrbound^2 \Id\eqsp;
    \]
    \item for the random matrix $\funcAw_1$ it holds that
    \[
    \PE[\funcAw_1^{\top} \G^{-1} \funcAw_1] \preceq \randmbound^2 \G\eqsp;
    \]
    \item for the matrix $\noisecov$ defined in \eqref{eq:def_noise_cov} it holds that
    \[
    \trace{\noisecov} \leq \tracebound^2 \trace{\G^{1/2} \bA^{-1} \noisecov \bA^{-T} \G^{1/2}}\eqsp;
    \]
\end{enumerate}
\end{assumptionC}
Under Assumption \Cref{assum:g_matr} we introduce a new notation
\[
\noisecovnew = \G^{1/2} \bA^{-1} \noisecov \bA^{-T} \G^{1/2}\eqsp.
\]
Our proof in this section follows the general procedure introduced for the Polyak-Ruppert estimator $\prtheta_{n_0,n}$ in \eqref{eq:pr_err_decompose}. Recall that with summation by parts  we obtain the following
\[
\bA\left(\prtheta_{n_0,n} -\thetalim\right) =  \frac{\theta_{n_0}-\theta_{n}}{\alpha (n - n_0)} - \frac{\sum_{t=n_0}^{n-1} e\left(\theta_{t}, Z_{t+1} \right)}{n-n_0} \eqsp,
\]
where the quantities $e\left(\theta_{t}, Z_{t+1} \right)$ are defined in \eqref{eq:pr_e_definition}. Since we assume that $\bA$ is non-degenerate, for symmetric positive-definite matrix $\G = \G^{\top} > 0$ from \Cref{assum:g_matr}, we get from the previous inequality that
\begin{equation}
\label{eq:sum_parts_refined}
\G^{1/2}\left(\prtheta_{n_0,n} -\thetalim\right) = \frac{\G^{1/2}\bA^{-1}(\theta_{n_0}-\theta_{n})}{\alpha (n - n_0)} - \frac{\G^{1/2}\bA^{-1} \sum_{t=n_0}^{n-1} e\left(\theta_{t}, Z_{t+1} \right)}{n-n_0}\eqsp.
\end{equation}

Based on the above identity, we prove the following counterpart of the $2$-nd-moment bound \Cref{th:lsa_pr_2nd_moment_main} for the general LSA problem.
\begin{theorem}
\label{th:lsa_pr_2nd_moment_refined}
Assume \Cref{assum:noise-level}, \Cref{assum:exp_stability}($2$), and \Cref{assum:g_matr}. Then for any $n \geq 2$, $\alpha \in (0;\alpha_{2,\infty}]$, it holds that
\begin{equation}
\label{eq:pr_theta_2nd_moment_bound_refined}
\begin{split}
\PE[\norm{\prtheta_{n} - \thetas}[\G]^2]
&\lesssim \frac{\trace{\noisecovnew}}{n} + \frac{\conststab[2]^2 \matrbound^2 \trace{\noisecov}}{a n} \left(\frac{ \norm{\G^{-1/2}}^2}{\alpha n} +  \randmbound^2\, \norm{\G^{1/2}}^2 \alpha \right) \\
&\qquad +\conststab[2]^2 \matrbound^2 (1 - \alpha a)^{n} \left(\frac{\norm{\G^{-1/2}}^2}{\alpha^2n^2} + \frac{\randmbound^2 \norm{\G^{1/2}}^2}{\alpha a n^2}\right)\norm{\theta_0 - \thetas}^2
\end{split}
\end{equation}
\end{theorem}
\begin{proof}
Following the pipeline of \Cref{th:lsa_pr_2nd_moment_main} and using \eqref{eq:sum_parts_refined}, we get
\[
\PE[\norm{\prtheta_{n} - \thetas}[\G]^2] \lesssim \underbrace{\frac{\PE[\norm{\G^{1/2}\bA^{-1}(\theta_{n/2} - \theta_n)}^2]}{\alpha^2 n^2}}_{T_1} \, + \, \underbrace{\frac{ \PE[\norm{\sum_{t=n/2}^{n-1} \G^{1/2}\bA^{-1} e\left(\theta_{t}, Z_{t+1} \right)}^2]}{n^2}}_{T_2}\eqsp,
\]
and estimate the terms $T_1$ and $T_2$ separately. Applying the bounds of \Cref{th:LSA_last_iterate_refined}, we get first that
\[
T_1 \lesssim \matrbound^2 \norm{\G^{-1/2}}^2 \conststab[2]^2 \, \biggl[ \frac{(1 - \alpha a)^{n}\norm{\theta_0 - \thetas}^2}{\alpha^2 n^2} + \frac{\trace{\noisecov}}{\alpha a n^2}\biggr]\eqsp.
\]
Here we additionally used an upper bound
\begin{equation}
\label{eq:norm_a_inv_bound}
\begin{split}
\norm{\G^{1/2} \bA^{-1} u}^2
&= u^{\top} \bA^{-\top} \G \bA^{-1} u \\
&= u^{\top} \G^{-1/2} \G^{1/2} \bA^{-\top} \G \bA^{-1} \G^{1/2} \G^{-1/2} u \\
& \leq \matrbound^2 u^{\top} \G^{-1} u \\
& \leq \matrbound^2 \norm{\G^{-1/2}}^2 \norm{u}^2\eqsp,
\end{split}
\end{equation}
which is valid for any $u \in \rset^{d}$. Similarly, since $\{ \G^{1/2}\bA^{-1} e\left(\theta_{t}, Z_{t+1} \right)\}_{t \in \nset}$ is a martingale-difference sequence w.r.t. filtration $\mathcal{F}_k = \sigma(Z_{j}, j \leq k)$, we get the following bound for $T_2$:
\begin{align}
T_2 
&\lesssim n^{-2}\sum_{t=n/2}^{n-1}\PE[\norm{\G^{1/2}\bA^{-1}e\left(\theta_{t}, Z_{t+1} \right)}^2] \\
&\lesssim \frac{\trace{\noisecovnew}}{n} + \conststab[2]^2 \randmbound^2 \matrbound^2\, \norm{\G^{1/2}}^2 \biggl[\frac{(1- \alpha a)^{n}\norm{\theta_0 - \thetas}^2}{\alpha a n^2} + \frac{\alpha \trace{\noisecov}}{an}\biggr]\eqsp.
\end{align}
In particular, to bound the first term we use the bound
\begin{align}
\label{eq:cond_f_t_bound}
&\CPE{\normop{\G^{1/2}\bA^{-1} \zmfuncAw[t+1](\theta_t - \thetas)}^{2}}{\mathcal{F}_t} \\
&\qquad = \CPE{(\theta_t - \thetas)^{\top} \zmfuncAw[t+1]^{\top} \bA^{-\top} \G \bA^{-1} \zmfuncAw[t+1](\theta_t - \thetas)}{\mathcal{F}_t} \\
&\qquad = \CPE{(\theta_t - \thetas)^{\top} \funcAw_{t+1}^{\top} \G^{-1/2} \G^{1/2}  \bA^{-\top} \G \bA^{-1} \G^{1/2} \G^{-1/2} \funcAw_{t+1} (\theta_t - \thetas)}{\mathcal{F}_t} \\
&\qquad \qquad \qquad \qquad \qquad \qquad \qquad \qquad \qquad \qquad \qquad \qquad - \CPE{(\theta_t - \thetas)^{\top} \G (\theta_t - \thetas)}{\mathcal{F}_t} \\
&\qquad \leq \matrbound^2 \CPE{(\theta_t - \thetas)^{\top} \funcAw_{t+1}^{\top} \G^{-1} \funcAw_{t+1} (\theta_t - \thetas)}{\mathcal{F}_t} \\
&\qquad \leq \randmbound^2 \matrbound^2 (\theta_t - \thetas)^{\top} \G (\theta_t - \thetas) \\
&\qquad \leq \randmbound^2 \matrbound^2\, \norm{\G^{1/2}}^2 \, \norm{\theta_t - \thetas}^2\eqsp.
\end{align}
\end{proof}

Now we trace \Cref{th:lsa_general_refined_norm} in the case of TD (0) updates. First we check whether the assumption \Cref{assum:g_matr} holds.

\begin{lemma}
\label{lem:new_conditions_td_check}
Let $\sequence{\theta}[k][\nset]$ be a sequence of TD(0) updates generated by \eqref{eq:LSA_procedure_TD} under \Cref{assum:generative_model} and \Cref{assum:feature_design}. Then this update scheme satisfies assumption \Cref{assum:g_matr} with
\begin{equation}
\label{eq:new_const_td}
\G = \covfeat\eqsp, \quad \matrbound = 1/(1-\gamma)\eqsp, \quad \randmbound = (1+\gamma)\lambda_{\min}^{-1/2}\eqsp, \quad \tracebound = 1+\gamma \eqsp.
\end{equation}
Moreover, it holds that
\begin{equation}
\label{eq:lower_bound_td_matr}
\covfeat^{-1/2} \bA^{\top} \covfeat^{-1} \bA \covfeat^{-1/2} \succeq (1-\gamma)^2 \Id\eqsp.
\end{equation}
\end{lemma}
\begin{proof}
In order to prove that
\[
\covfeat^{1/2} \bA^{-\top} \covfeat \bA^{-1} \covfeat^{1/2} \preceq \frac{1}{(1-\gamma)^2} \Id\eqsp,
\]
it is enough to show the lower bound \eqref{eq:lower_bound_td_matr}. For the finite state space $\S$ this follows from \cite[Lemma~5]{li2023sharp}, we provide a slightly modified argument for completeness. Indeed, for any $x \in \rset^{d}$, using that $\bA = \covfeat - \gamma \PE[\varphi(s_1) \varphi(s_1^{\prime})^{\top}]$, we have
\begin{align}
x^{\top} \covfeat^{-1/2} \bA^{\top} \covfeat^{-1} \bA \covfeat^{-1/2} x
&= \norm{\covfeat^{-1/2} \bA \covfeat^{-1/2} x}^2 = \norm{\bigl(\Id - \gamma \covfeat^{-1/2} \PE[\varphi(s_1) \varphi(s_1^{\prime})^{\top}] \covfeat^{-1/2}\bigr) x}^2 \\
&\geq (\norm{x} - \gamma \norm{\covfeat^{-1/2} \PE[\varphi(s_1) \varphi(s_1^{\prime})^{\top}] \covfeat^{-1/2} x})^2\eqsp,
\end{align}
and to complete the proof it is enough to show that $\norm{\covfeat^{-1/2} \PE[\varphi(s_1) \varphi(s_1^{\prime})^{\top}] \covfeat^{-1/2}} \leq 1$. In order to do it, note that
\begin{align}
\label{eq:non-sym-bound}
\norm{\covfeat^{-1/2} \PE[\varphi(s_1) \varphi(s_1^{\prime})^{\top}] \covfeat^{-1/2}}
&= \sup_{ \norm{x} = 1, \norm{y} = 1} x^{\top} \covfeat^{-1/2} \PE[\varphi(s_1) \varphi(s_1^{\prime})^{\top}] \covfeat^{-1/2} y \\
&=\sup_{ \norm{x} = 1, \norm{y} = 1} \PE\left[ \left( [\covfeat^{-1/2} x]^\top \varphi(s_1) \right) \left(\varphi(s_1^{\prime})^{\top} \covfeat^{-1/2} y\right)\right] \\
&\leq \sup_{ \norm{x} = 1, \norm{y} = 1} \PE\left[ \frac{1}{2}\left( [\covfeat^{-1/2} x]^\top \varphi(s_1) \right)^2 +  \frac{1}{2}\left(\varphi(s_1^{\prime})^{\top} \covfeat^{-1/2} y\right)^2\right] \\
&= \sup_{ \norm{x} = 1, \norm{y} = 1}\biggl[\frac{1}{2}x^{\top} \covfeat^{-1/2} \PE[\varphi(s_1) \varphi(s_1)^{\top}] \covfeat^{-1/2} x \\
&\qquad\qquad\qquad\quad + \frac{1}{2}y^{\top} \covfeat^{-1/2} \PE[\varphi(s_1^{\prime}) \varphi(s_1^{\prime})^{\top}] \covfeat^{-1/2} y \biggr] = 1.
\end{align}
where we used the fact that a distribution $\mu$ is the invariant. Hence, we get
\[
x^{\top} \covfeat^{-1/2} \bA^{\top} \covfeat^{-1} \bA \covfeat^{-1/2} x \geq (1-\gamma)^2\norm{x}^2\eqsp,
\]
and the bound \eqref{eq:lower_bound_td_matr} is proved. To prove the second part of the bound, we use \eqref{eq:ata_bound} and obtain that
\begin{align}
\PE[\funcAw_1^{\top} \covfeat^{-1} \funcAw_1]
&= \PE[\bigl(\varphi(s_1) - \gamma \varphi(s_1^{\prime})\bigr)\varphi(s_1)^{\top} \covfeat^{-1} \varphi(s_1) \bigl(\varphi(s_1) - \gamma \varphi(s_1^{\prime})\bigr)^{\top}] \\
&\preceq \lambda_{\min}^{-1} \PE[\bigl(\varphi(s_1) - \gamma \varphi(s_1^{\prime})\bigr)\bigl(\varphi(s_1) - \gamma \varphi(s_1^{\prime})\bigr)^{\top}] \preceq \lambda_{\min}^{-1} (1+\gamma)^2 \covfeat\eqsp,
\end{align}
where the last inequality follows \eqref{eq:ata_bound} in the proof of \Cref{lem:B_matr_bounds}. To check the last one, note that
\begin{align}
\trace{\noisecov}
&= \trace{ \bA^{\top} \covfeat^{-1/2} \covfeat^{1/2} \bA^{-\top} \noisecov \bA^{-1} \covfeat^{1/2} \covfeat^{-1/2} \bA } \\
& \overset{(a)}{=} \trace{ \covfeat^{-1/2} \bA \covfeat^{-1/2} \covfeat \covfeat^{-1/2} \bA^{\top} \covfeat^{-1/2} \covfeat^{1/2} \bA^{-\top} \noisecov \bA^{-1} \covfeat^{1/2} } \\
& \overset{(b)}{\leq} \norm{\covfeat^{-1/2} \bA \covfeat^{-1/2} \covfeat \covfeat^{-1/2} \bA^{\top} \covfeat^{-1/2}} \trace{ \covfeat^{1/2} \bA^{-\top} \noisecov \bA^{-1} \covfeat^{1/2} }\eqsp.
\end{align}
Note that the identity (a) above follows from the cyclic property of trace, and the inequality (b) is due to $\trace{CD} \leq \norm{C}\trace{D}$, which is valid for symmetric positive semi-definite matrices $C,D$. In the bound above it remains to estimate
\begin{align}
\label{eq:first-ineq}
\norm{\covfeat^{-1/2} \bA \covfeat^{-1/2} \covfeat \covfeat^{-1/2} \bA^{\top} \covfeat^{-1/2}} \leq \norm{\covfeat^{-1/2} \bA \covfeat^{-1/2}}^2 \norm{ \covfeat }\eqsp.
\end{align}
Consider now the operator norm of the matrix $ \covfeat^{-1/2} \bA \covfeat^{-1/2} $. Note that $\bA = \covfeat - \gamma \PE[\varphi(s_1) \varphi(s_1^{\prime})^{\top}]$. Thus,  we get
\begin{align}
&\norm{ \covfeat^{-1/2} \bA \covfeat^{-1/2} } = \sup_{ \norm{x} = 1, \norm{y} = 1} x^{\top} \covfeat^{-1/2} (\covfeat - \gamma \PE[\varphi(s_1) \varphi(s_1^{\prime})^{\top}]) \covfeat^{-1/2} y \\
&\quad \leq 1 + \gamma \sup_{ \norm{x} = 1, \norm{y} = 1 }\biggl[\frac{x^{\top} \covfeat^{-1/2} \PE[\varphi(s_1) \varphi(s_1)^{\top}] \covfeat^{-1/2} x}{2} + \frac{y^{\top} \covfeat^{-1/2} \PE[\varphi(s_1^{\prime}) \varphi(s_1^{\prime})^{\top}] \covfeat^{-1/2} y}{2} \biggr] \\
&\quad = 1 + \gamma\eqsp.
\end{align}
Plugging this inequality into \eqref{eq:first-ineq}, we get
\begin{align}
\norm{\covfeat^{-1/2} \bA \covfeat^{-1/2} \covfeat \covfeat^{-1/2} \bA^{\top} \covfeat^{-1/2}}  \leq (1+\gamma)^2\eqsp.
\end{align}
In the last bound we additionally used that $\norm{ \covfeat } \leq 1$ under \Cref{assum:feature_design}.
\end{proof}

Now a simple combination of the above bounds allows us to prove the following result:
\begin{theorem}
\label{th:lower_bound_instance_2nd_moment}
Assume \Cref{assum:generative_model} and \Cref{assum:feature_design}.
Let $\sequence{\theta}[k][\nset]$ be a sequence of TD(0) updates generated by \eqref{eq:LSA_procedure_TD}.
Then for any $p \geq 2$, $n \geq 2$, $\alpha \in \bigl(0; \frac{1-\gamma}{256}\bigr]$, it holds that
\begin{equation}
\label{eq:pr_theta_2_moment_bound_new}
\begin{split}
\PE[\norm{\prtheta_{n} - \thetas}[\covfeat]^2]
&\lesssim \frac{\trace{\covfeat^{1/2} \bA^{-1} \noisecovtd \bA^{-T} \covfeat^{1/2}}}{n} + \frac{1+\norm{\thetas}[\covfeat]^2}{(1-\gamma)^3 \lambda_{\min}^2 n} \left(\frac{1}{\alpha n} + \alpha \right) \\
&\qquad + \frac{(1 - \alpha (1-\gamma) \lambda_{\min}/2)^{n}}{\lambda_{\min} (1-\gamma)^2} \left(\frac{1}{\alpha^2 n^2} + \frac{1}{\alpha (1-\gamma) \lambda_{\min} n^2}\right)\norm{\theta_0 - \thetas}^2
\end{split}
\end{equation}
\end{theorem}

Based on the identity above, we can prove the following counterpart of the result \Cref{th:lsa_pr_error_p_moment} for the general LSA problem.

\begin{theorem}
\label{th:lsa_general_refined_norm}
Assume \Cref{assum:noise-level}, \Cref{assum:exp_stability}($\infty$), and \Cref{assum:g_matr}. Then for any $p \geq 2$, $n \geq 2$, $\alpha \in \coint{0;\alpha_{p+\log{n},\infty}}$, it holds that
\begin{equation}
\label{eq:pr_theta_p_moment_bound}
\begin{split}
\PE^{1/p}[\norm{\prtheta_{n} - \thetas}[\G]^p]
&\lesssim \frac{p^{1/2}\sqrt{\trace{\noisecovnew}}}{n^{1/2}}\left(1 + \frac{ \sqrt{\alpha p} \conststab[\infty] \randmbound \matrbound \tracebound \norm{\G^{1/2}} \bConst{A}}{\sqrt{a}} + \frac{\alpha p \conststab[\infty] \randmbound \matrbound \norm{\G^{1/2}} \supconsteps}{\sqrt{\trace{\noisecovnew}}} \right)  \\
&+ \frac{p^{1/2} \matrbound \tracebound \normop{\G^{-1/2}} \conststab[\infty] \sqrt{\trace{\noisecovnew}}}{\sqrt{a} n} \biggl[ \frac{1}{\sqrt{\alpha}} + p^{1/2} \bConst{A} \sqrt{\alpha (p + \log n)}\biggr] \\
&+ \frac{p \matrbound \normop{\G^{-1/2}} \conststab[\infty] \supconsteps}{n}\left(1 + \bConst{A} \alpha (p + \log n) \right) \\
& + \matrbound \normop{\G^{-1/2}} \conststab[\infty] (1- \alpha a)^{n/2} \norm{\theta_0 - \thetas} \left(\frac{1}{\alpha n} + \frac{p \bConst{A}}{n} + \frac{p^{1/2} \randmbound}{\sqrt{\alpha a} n} \right) \eqsp.
\end{split}
\end{equation}
\end{theorem}
\begin{proof}
The proof follows the general scheme of \Cref{th:lsa_pr_error_p_moment}. Setting $n_0 = n/2$ and using Minkowski's inequality, we obtain from \eqref{eq:sum_parts_refined} that
\begin{multline}
\label{eq:p_moment_iid_decomp_matrix_norm}
\PE^{1/p}\left[\norm{\prtheta_{n} - \thetas}[\G]^{p}\right] \leq \underbrace{\frac{\PE^{1/p}[\norm{\G^{1/2}\bA^{-1}(\theta_{n/2}-\theta_{n})}^p]}{\alpha n}}_{T_1} + \\ 
\underbrace{\frac{\PE^{1/p}\bigl[\norm{\G^{1/2}\bA^{-1} \sum_{t=n/2}^{n-1}\rme(\theta_{t},\State_{t+1})}^{p}\bigr]}{n}}_{T_2}\eqsp,
\end{multline}
and bound $T_1$, $T_2$ separately. We begin with bounding the term $T_1$, which is a remainder term (w.r.t. sample size $n$). With \Cref{th:LSA_last_iterate_refined}-\eqref{eq:Rosenthal_LSA},  \Cref{assum:g_matr}, and \eqref{eq:norm_a_inv_bound}, we obtain
\[
T_1 \lesssim \matrbound \normop{\G^{-1/2}} \conststab[\infty] \biggl[ \frac{ (1- \alpha a)^{n/2} \norm{\theta_0 - \thetas}}{\alpha n} + \frac{ p^{1/2} \tracebound \sqrt{\trace{\noisecovnew}}}{\sqrt{\alpha a} n} + \frac{p \supconsteps}{n}\biggr]\eqsp.
\]
Now we bound $T_2$. Using again Minkowski's inequality, we get
\[
\textstyle
T_2 \leq n^{-1}\, \PE^{1/p}\bigl[\norm{\sum_{t=n/2}^{n-1}\G^{1/2}\bA^{-1} \funnoisew_{t+1}}^{p}\bigr] + n^{-1}\, \PE^{1/p}\parentheseDeuxLigne{\norm{\sum_{t=n/2}^{n-1} \G^{1/2}\bA^{-1} \zmfuncAw[t+1](\theta_t - \thetas)}^p} \eqsp.
\]
The first term of the above sum can be controlled by directly applying Pinelis' version of Rosenthal's inequality \cite[Theorem~4.3]{pinelis_1994}:
\begin{align}
\PE^{1/p}\left[\bigg\Vert\sum_{t=n/2}^{n-1}\G^{1/2}\bA^{-1}\funnoisew_{t+1} \bigg\Vert^{p}\right]
&\lesssim p^{1/2} n^{1/2} \sqrt{\trace{\G^{1/2} \bA^{-1} \noisecov \bA^{-T} \G^{1/2}}} + p \norm{\G^{1/2} \bA^{-1} \varepsilon}[\infty] \\
&\overset{(a)}{\lesssim} p^{1/2} n^{1/2} \sqrt{\trace{\noisecovnew}} + p \matrbound \normop{\G^{-1/2}}\norm{\varepsilon}[\infty]\eqsp.
\end{align}
In order to prove the step (a) above we used the bound \eqref{eq:norm_a_inv_bound}. Hence it remains to bound the quantity 
\[
\PE^{1/p}\parentheseDeuxLigne{\norm{\sum_{t=n/2}^{n-1} \G^{1/2}\bA^{-1} \zmfuncAw[t+1](\theta_t - \thetas)}^p}\eqsp.
\]
Note that $\{\G^{1/2}\bA^{-1} \zmfuncAw[t+1](\theta_t - \thetas)\}$ is a martingale-difference w.r.t. $\mathcal{F}_t = \sigma(Z_k, k \leq t)$. A further application of Rosenthal's inequality thus shows that
\begin{multline}
\label{eq:2nd_term_rosenthal_normed}
\PE^{1/p}\biggl[\bigg\Vert \sum_{t=n/2}^{n-1} \G^{1/2}\bA^{-1} \zmfuncAw[t+1](\theta_t - \thetas) \bigg\Vert^p\biggr] \lesssim
p^{1/2} \PE^{1/p}\biggl[\biggl(\sum_{t=n/2}^{n-1} \CPE{\normop{\G^{1/2}\bA^{-1} \zmfuncAw[t+1](\theta_t - \thetas)}^{2}}{\mathcal{F}_t}\biggr)^{p/2}\biggr] \\
+ p\, \PE^{1/p}\left[\max_{t}\normop{ \G^{1/2}\bA^{-1} \zmfuncAw[t+1](\theta_t - \thetas)}^{p}\right]\eqsp.
\end{multline}
Now we upper bound both terms in the r.h.s. separately. Using the bound \eqref{eq:cond_f_t_bound}, for the first term in r.h.s. of \eqref{eq:2nd_term_rosenthal_normed} we have, using \Cref{th:LSA_last_iterate_refined}-\eqref{eq:Rosenthal_LSA} and \Cref{assum:g_matr}, that
\begin{align}
& p^{1/2} \PE^{1/p}\biggl[\bigl(\sum_{t=n/2}^{n-1} \CPE{\normop{\G^{1/2}\bA^{-1} \zmfuncAw[t+1](\theta_t - \thetas)}^{2}}{\mathcal{F}_t}\bigr)^{p/2}\biggr] \leq p^{1/2} \randmbound \matrbound \norm{\G^{1/2}} \biggl[\sum_{t=n/2}^{n-1} \PE^{2/p}\left[\norm{\theta_t - \thetas}^{p}\right]\biggr]^{1/2} \\
& \qquad \lesssim \conststab[\infty] p^{1/2} \randmbound \matrbound \norm{\G^{1/2}} \cdot \biggl[ \frac{(1- \alpha a)^{n/2} \norm{\theta_0 - \thetas}}{\sqrt{\alpha a}} +   \frac{p^{1/2} \tracebound \sqrt{\alpha n \trace{\noisecovnew}}}{\sqrt{a}} + \alpha p n^{1/2} \supconsteps \biggr] \eqsp.
\end{align}
For the second term in \eqref{eq:2nd_term_rosenthal_normed} we have, applying \Cref{th:LSA_last_iterate_refined}-\eqref{eq:Rosenthal_LSA} and using $n^{1/\log{n}} \leq \rme$, that
\begin{align}
&p\, \PE^{1/p}[\max_{t}\normop{ \G^{1/2}\bA^{-1} \zmfuncAw[t+1](\theta_t - \thetas)}^{p}] \\
&\quad \lesssim p \bConst{A} \matrbound \norm{\G^{-1/2}} n^{1/(p+\log n)} \max_{n/2 \leq t < n}  \PE^{1/(p+\log{n})}[\normop{\theta_t - \thetas}^{p+\log{n}}] \\
&\quad \lesssim \conststab[\infty] p \bConst{A} \matrbound \norm{\G^{-1/2}} \cdot \biggl[ (1- \alpha a)^{n/2} \norm{\theta_0 - \thetas} + \frac{\sqrt{\alpha (p + \log n) \trace{\noisecovnew}}}{a} + \alpha (p + \log n) \supconsteps \biggr] \eqsp.
\end{align}
Now the statement follows from combining the above estimates in \eqref{eq:p_moment_iid_decomp_matrix_norm}.
\end{proof}

Now a simple combination of the above bounds allows us to prove the following bound:
\begin{theorem}
\label{th:lower_bound_instance}
Assume \Cref{assum:generative_model} and \Cref{assum:feature_design}.
Let $\sequence{\theta}[k][\nset]$ be a sequence of TD(0) updates generated by \eqref{eq:LSA_procedure_TD}.
Then for any $p \geq 2$, $n \geq 2$, $\alpha \in \bigl(0; \frac{1-\gamma}{128(p + \log{n})}\bigr]$, it holds that
\begin{equation}
\label{eq:pr_theta_p_moment_bound}
\begin{split}
&\PE^{1/p}[\norm{(\prtheta_{n} - \thetas)}[\covfeat]^p]
\lesssim \frac{p^{1/2}\sqrt{\trace{\covfeat^{1/2} \bA^{-1} \noisecovtd \bA^{-T} \covfeat^{1/2}}}}{n^{1/2}}\left(1 + \frac{ \sqrt{\alpha p} }{(1-\gamma)^{3/2} \lambda_{\min}} \right) + \frac{\alpha p^{3/2} (1+\norm{\thetas})}{(1-\gamma) \lambda_{\min}^{1/2} n^{1/2}} \\
&+ \frac{p^{1/2} \sqrt{\trace{\covfeat^{1/2} \bA^{-1} \noisecovtd \bA^{-T} \covfeat^{1/2}}}}{(1-\gamma)^{3/2} \lambda_{\min} n} \biggl[ \frac{1}{\sqrt{\alpha}} + p^{1/2} \sqrt{\alpha (p + \log n)}\biggr] + \frac{p (1+\norm{\thetas})}{(1-\gamma) \lambda_{\min} n}\biggl[1 + \alpha (p + \log n) \biggr] \\
& + \frac{1}{(1-\gamma) \lambda_{\min}^{1/2}} \biggl(1- \alpha (1-\gamma) \lambda_{\min}\biggr)^{n/2} \norm{\theta_0 - \thetas} \left(\frac{1}{\alpha n} + \frac{p}{n} + \frac{p^{1/2}}{\sqrt{\alpha} (1-\gamma) \lambda_{\min}^{1/2} n} \right) \eqsp.
\end{split}
\end{equation}
\end{theorem}

\section{Berbee's lemma and coupling inequalities for Markov chains}
\label{sec:rosenthal_markov}
We preface this section with some definitions essential for the Berbee's lemma construction. Consider a probability space $(\Omega,\mathcal{F},\PP)$ equipped with $\sigma$-fields $\mathfrak{F}$ and $\mathfrak{G}$ such that $\mathfrak{F} \subseteq \mathcal{F}\,, \mathfrak{G} \subseteq \mathcal{F}$. Then the $\beta$-mixing coefficient of $\mathfrak{F}$ and $\mathfrak{G}$ is defined as
\begin{equation}
\label{eq:beta_mixing}
\beta(\mathfrak{F},\mathfrak{G}) = (1/2) \sup \sum_{i \in \msi} \sum_{j \in \msj} | \PP ( \msa_i \cap \msb_j)- \PP(\msa_i)\PP(\msb_j)|\eqsp,
\end{equation}
and the supremum is taken over all pairs of partitions $\{\msa_i\}_{i\in\msi} \in \mathfrak{F}^\msi$ and $\{\msb_j\}_{j\in\msj}\in \mathfrak{G}^\msj$ of $\tmszn$ with finite $\msi$ and $\msj$.
\par
Now let $(\msz,\metricz)$ be a Polish space endowed with its Borel $\sigma$-field, denoted by $\mcz$, and let $(\msz^{\nset}, \mcz^{\otimes \nset})$ be the corresponding canonical space. Consider a Markov kernel $\MKQ$ on $\msz\times \mcz$ and denote by $\PP_{\xi}$ and $\PE_{\xi}$ the corresponding probability distribution and expectation with initial distribution $\xi$. Without loss of generality, we assume that $(Z_k)_{k \in \nset}$ is the associated canonical process. By construction, for any $\msa \in \mcz$, $\CPP[\xi]{Z_k \in \msa}{Z_{k-1}}= \MKQ(Z_{k-1},\msa)$, $\PP_\xi$-a.s. In the case $\xi= \updelta_z$, $z \in \msz$, $\PP_{\xi}$ and $\PE_{\xi}$ are denoted by $\PP_{z}$ and $\PE_{z}$, respectively. We now make an assumption about the mixing properties of $\MKQ$, which essentially reflects \Cref{assum:P_pi_ergodicity}.
\begin{assumptionM}
\label{assum:uge}
The Markov kernel $\MKQ$ admits $\invariantQ$ as an invariant distribution and is uniformly geometrically ergodic, that is, there exists $\taumix \in \nsets$ such that for all $k \in \nset$,
\begin{equation}
\label{eq:drift-condition}
\dobru{\MKQ^k} = \sup_{z,z' \in \Zset} (1/2) \norm{\MKQ^k(z, \cdot) - \MKQ^k(z',\cdot)}[\mathsf{TV}] \leq (1/4)^{\lfloor k / \taumix \rfloor} \eqsp.
\end{equation}
\end{assumptionM}

For $q \in \nset$, $k \in \nset$, and the Markov chain $\{Z_n\}_{n \in \nset}$ satisfying the uniform geometric ergodicity condition \Cref{assum:uge}, we define the $\sigma$-algebras $\mathcal{F}_{k} = \sigma(Z_{\ell}, \ell \leq k)$ and $\mathcal{F}^{+}_{k+q} = \sigma(Z_{\ell}, \ell \geq k+q)$. In such a scenario, using \cite[Theorem~3.3]{douc:moulines:priouret:soulier:2018}, the respective $\beta$-mixing coefficient of $\mathcal{F}_{k}$ and $\mathcal{F}^{+}_{k+q}$ is bounded by
\begin{equation}
\label{eq:beta_q_mixing_def}
\beta(q) \equiv \beta(\mathcal{F}_{k},\mathcal{F}^{+}_{k+q}) \leq (1/4)^{\lfloor q / \taumix \rfloor} \eqsp.
\end{equation}
In this chapter we rely on the following useful version of Berbee's coupling lemma \cite{berbee1979}, which is due to \cite[Lemma~$4.1$]{dedecker2002maximal}:

\begin{lemma} (Lemma~$4.1$ in \cite{dedecker2002maximal})
\label{lem:Dedecker_lemma}
Let $X$ and $Y$ be two random variables taking their values in Borel spaces $\mathcal{X}$ and $\mathcal{Y}$, respectively, and let $U$ be a random variable with uniform distribution on $[0;1]$ that is independent of $(X,Y)$. There exists a random variable $Y^{\star} = f(X,Y,U)$ where $f$ is a measurable function
 from $\mathcal{X} \times \mathcal{Y} \times [0,1]$ to $\mathcal{Y}$, such that:
\begin{enumerate}
\item $Y^{\star}$ is independent of $X$ and has the same distribution as $Y$;
\item $\PP(Y^\star \neq Y)= \beta(\sigma(X),\sigma(Y))$.
\end{enumerate}
\end{lemma}

Let us now consider the extended measurable space $\tmszn = \msz^{\nset} \times [0,1]$, equipped with the $\sigma$-field $\tmczn =\mcz^{\otimes \nset} \otimes \mathcal{B}([0,1])$. For each probability measure $\xi$ on $(\Zset,\Zsigma)$, we consider the probability measure $\PPext_{\xi} =\PP_{\xi} \otimes \mathbf{Unif}([0,1])$ and denote by $\PEext_{\xi}$ the corresponding expected value. Finally, we denote by $(\tZ_k)_{k \in\nset}$ the canonical process $\tZ_k\colon ((z_i)_{i\in\nset},u) \in \tmszn \mapsto z_k$ and $U \colon((z_i)_{i\in\nset},u) \in \tmszn \mapsto u$. Under $\PPext_{\xi}$, $\sequence{\tZ}[k][\nset]$ is by construction a Markov chain with initial distribution $\xi$ and Markov kernel $\MKQ$ independent of $U$. Moreover, the distribution of $U$ under $\PPext_{\xi}$ is uniform over $\ccint{0,1}$. Using the above construction, we obtain a useful blocking lemma, which is also stated in \cite{dedecker2002maximal}.

\begin{lemma}
\label{lem:Dedecker_upd}
Assume \Cref{assum:uge}, let $q \in \nsets$ and $\xi$ be a probability measure on $(\msz,\mcz)$.   Then,  there exists a random process $(\tZs_{k})_{k\in\nset}$ defined on $(\tmszn, \tmczn, \PPext_{\xi})$ such that for any $k \in \nset$,
  \begin{enumerate}[wide,label=(\alph*)]
  \item For any $i$, vector $V_{i}^{\star} = (\tZs_{iq+1},\ldots,\tZs_{iq+q})$ has the same distribution as $V_{i} = (Z_{iq+1},\ldots,Z_{iq+q})$ under $\PPext_{\xi}$;
  \item The sequences $(V_{2i}^{\star})_{i \geq 0}$ and $(V_{2i+1}^{\star})_{i \geq 0}$ are \iid\ ;
  \item For any $i$, $\PPext_{\xi}(V_{i} \neq V_{i}^{\star}) \leq  \beta(q)$;
  \end{enumerate}
\end{lemma}
\begin{proof}
The proof follows from \Cref{lem:Dedecker_lemma} and the relations between \Cref{assum:uge} and $\beta$-mixing coefficient, see \cite[Theorem~3.3]{douc:moulines:priouret:soulier:2018}.
\end{proof}

\subsection{Proof of \Cref{cor:td_pr_deviation_markov}}
\label{sec:proof_markov}
We aim to reduce the proof of the given bound to that of \Cref{cor:td_pr_pth_moment}. Since the initial distribution of the sequence of states is $s_0 \sim \nu$, we must first remove the dependence on the initial condition. Indeed, using \cite[Lemma~19.3.6 and Theorem~19.3.9 ]{douc:moulines:priouret:soulier:2018} for any two probabilities $\nu$ and $\tilde{\nu}$ on $(\S,\borel{\S})$ there is a \emph{maximal exact coupling} $(\Omega,\mathcal{F},\PPcoupling{\nu}{\tilde{\nu}},s,\tilde{s},T)$ of $\PP_{\nu}$ and $\PP_{\tilde{\nu}}$, that is,
\begin{equation}
\label{eq:coupling_time_def_markov}
\textstyle{
\tvnorm{\nu \PMDP_{\pi}^{n} - \tilde{\nu}\PMDP_{\pi}^{n}} = 2 \PPcoupling{\nu}{\tilde{\nu}}(T > n)
}\eqsp.
\end{equation}
Under $\PPcoupling{\nu}{\tilde{\nu}}$, the sequences $\{s_k\}_{k \in \nset}$ and $\{\tilde{s}_k\}_{k \in \nset}$ are Markov chains with initial distributions $\nu$ and $\tilde{\nu}$, respectively. We write $\PEcoupling{\nu}{\tilde{\nu}}$ for the expectation with respect to $\PPcoupling{\nu}{\tilde{\nu}}$.  $T$ is the coupling time, which is defined as
\begin{equation}
\label{eq:coupling_time}
T = \inf_{k \in \nset}\{s_k = \tilde{s}_k\}\eqsp.
\end{equation}
Let us now fix $\tilde{\nu} = \mu$ and for $n \in \nset$ define an event $A_{n} = \{T > n/2\}$. Under \Cref{assum:P_pi_ergodicity}, we can bound its probability as
\[
\PPcoupling{\nu}{\mu}(A_{n}) = \PPcoupling{\nu}{\mu}(T > n/2) \leq (1/4)^{\lfloor n / (2\taumix) \rfloor}\eqsp.
\]
Thus, for a fixed $\delta \in (0;1/3)$ we can achieve $\PPcoupling{\nu}{\mu}(A_{n}) \leq \delta$ as soon as
\[
n \geq \frac{2 \taumix \log(4/\delta)}{\log{4}}\eqsp.
\]
Hence, starting from this point we work on the event $\Omega \setminus A_{n}$ which has probability at least $1 - \delta$. On this event $\{s_k\}_{k \geq n/2}$ coincides with $\{\tilde{s}_k\}_{k \geq n/2}$, which is a stationary Markov chain with initial distribution $\mu$. Assume now that the sample size $n$ satisfies
\begin{equation}
\label{eq:sample_size_constraint}
n/2 = 2 q m + k\eqsp, \quad 0 \leq k < 2q\eqsp,
\end{equation}
where $q \in \nset$ is a parameter that will be determined later. Using the construction of \Cref{lem:Dedecker_upd}, we then construct a sequence of random variables$\{\tilde{s}^{\star}_{n/2 + 2 j q}\}_{j = 1,\ldots,m}$, which are \iid\ with law $\mu$ under $\PPcoupling{\nu}{\mu}$. Moreover, with a union bound,
\[
\PPcoupling{\nu}{\mu}(\exists j \in \{1,\ldots,m\}: \tilde{s}^{\star}_{n/2 + 2 j q} \neq \tilde{s}_{n/2 + 2 j q}) \leq m (1/4)^{\lfloor q / \taumix \rfloor}\eqsp.
\]
The bound \eqref{eq:sample_size_constraint} implies that $m \leq n / (4q)$. Thus in order to achieve that $\PPcoupling{\nu}{\mu}(\exists j \in \{1,\ldots,m\}: \tilde{s}^{\star}_{n/2 + 2 j q} \neq \tilde{s}_{n/2 + 2 j q}) \leq \delta$ it is enough to ensure that
\[
m (1/4)^{\lfloor q / \taumix \rfloor} \leq 4m (1/4)^{q / \taumix} \leq (n/q)  (1/4)^{q / \taumix} \leq \delta\eqsp.
\]
In order to satisfy this constraint for fixed $\delta \in (0,1)$, we choose
\begin{equation}
\label{eq:q_block_size}
q = \left\lceil \frac{\taumix \log{(n/\delta)}}{\log{4}}\right\rceil\eqsp.
\end{equation}
Thus, setting the block size $q$ as in \eqref{eq:q_block_size}, we get that for sample size $n$ satisfying \eqref{eq:sample_size_constraint}, with probability at least $1-2\delta$ the results of Algorithm~\ref{alg:TD_skip} are indistinguishable from the result of Algorithm~\ref{alg:TD_iid} under the generative model assumption \Cref{assum:generative_model} applied with sample size
\[
n / (4q) - 1 \leq m \leq n / (4q)\eqsp.
\]
Hence, the rest of the proof follows directly from the results of \Cref{cor:td_pr_pth_moment} applied with sample size $m$.